\documentclass{article}

\PassOptionsToPackage{authoryear,round}{natbib}

\usepackage[main, final]{neurips_2026}
\usepackage[utf8]{inputenc} % allow utf-8 input
\usepackage[T1]{fontenc}    % use 8-bit T1 fonts
\usepackage{microtype}
\usepackage{graphicx}
\graphicspath{{./figs/}}
\usepackage{subcaption}
\usepackage{booktabs} % for professional tables
\usepackage{multirow} % for multirow cells in tables
\usepackage{nicefrac} % compact symbols for 1/2, etc.

\usepackage{hyperref}
\hypersetup{hidelinks}
\usepackage{url}

\usepackage{amsmath}
\usepackage{amssymb}
\usepackage{mathtools}
\usepackage{algorithm}
\usepackage{algorithmic}
\usepackage{amsthm}
\usepackage{booktabs}
\usepackage{multirow}
\usepackage{makecell}
\usepackage{siunitx}
\usepackage{tabularx}
\usepackage{adjustbox}
\usepackage{colortbl}
\usepackage{rotating}
\usepackage[table]{xcolor}
\usepackage{caption}
\usepackage{placeins}
\usepackage{float}
\usepackage{pifont}
\newcommand{\cmark}{\ding{51}}%
\newcommand{\xmark}{\ding{55}}%

\definecolor{linkblue}{RGB}{46, 116, 181}
\definecolor{cherryred}{RGB}{196, 36, 70}

\hypersetup{
  colorlinks=true,
  citecolor=linkblue,
  urlcolor=cherryred,
  linkcolor=cherryred
}

\newcommand{\ours}{\textsc{CompilerKV}}
\newcommand{\itbf}[1]{\textit{\textbf{#1}}}

\usepackage[capitalize,noabbrev]{cleveref}

\theoremstyle{plain}
\newtheorem{theorem}{Theorem}[section]
\newtheorem{proposition}[theorem]{Proposition}

\theoremstyle{definition}

\theoremstyle{remark}

\usepackage[disable,textsize=tiny]{todonotes}

\title{\textsc {CompilerKV}: Risk-Adaptive KV Cache Compression via Offline Experience Compilation}

\author{%
  Ning Yang$^{1}$\thanks{Corresponding author: Ning Yang (\texttt{ning.yang@ia.ac.cn})} \quad
  Chengzhi Wang$^{2}$
  \quad
  Yibo Liu$^{3}$\quad
  Baoliang Tian$^{4}$\quad
  Haijun Zhang$^{5}$\\[0.6em]
  {\normalfont\small $^{1}$ Institute of Automation, Chinese Academy of Sciences}\\
  {\normalfont\small $^{2}$ University of Electronic Science and Technology of China}\\
  {\normalfont\small $^{3}$ The Chinese University of Hong Kong, Shenzhen}\quad
  {\normalfont\small $^{4}$ Tianjin University} \\
  {\normalfont\small $^{5}$ University of Science and Technology Beijing}
}

\begin{document}

\maketitle
% Keep the official NeurIPS style parameters unchanged.

\begin{abstract}
Prefill-only KV compression freezes a token subset at the end of prefill and decodes from it without further eviction. The retention decision is therefore irreversible, yet existing methods estimate the corrective signals it relies on, per-head reliability and prompt-level compression sensitivity, online from a single noisy prompt. We argue this is the wrong statistical unit: these signals exhibit far higher cross-prompt regularity than within-prompt signal-to-noise. We introduce \textsc{CompilerKV}, a KV-retention policy whose corrective tables are compiled offline from a calibration corpus, reducing online correction after the standard observation-window scan to $O(1)$ lookups plus a budget clamp. We find that compiled retention tables behave as portable architectural priors: rankings transfer across disjoint corpora on four backbones (mean Spearman $\bar\rho{=}0.90$), and direct model-to-model table transfer costs only $0.4$--$0.8$ LongBench points on average. At a 512-token budget, \textsc{CompilerKV} attains compressed-SOTA on all four backbones, improving over the strongest prefill-only baseline by $+1.67$ points on average (task-bootstrap 95\% CI $[+1.08,+2.37]$). Pressure regimes amplify the gap: under a fixed $512/32k$ cache ratio, \textsc{CompilerKV} remains the strongest compressed method through 128k RULER ($\sim\!73$ vs.\ FullKV $\sim\!79$, SnapKV $\sim\!38$); on 32k NIAH it reaches $0.89$ vs.\ SnapKV $0.42$; and at 32k input, retaining only $1.56\%$ of the prefill KV, batch-16 serving remains feasible where FullKV is OOM.
Code is available at \url{https://github.com/Chengzhi-W/CompilerKV}.
\end{abstract}

\section{Introduction}

Large language models (LLMs) are bottlenecked by the linear growth of the Key--Value (KV) cache: a 7B model at 128k context already exceeds 100\,GB~\citep{pope2023efficiently, xiao2024duoattention, behrouz2024titans}. \emph{Prefill-only} KV compression, which selects a token subset once at the end of prefill and freezes it for the entire decoding pass, has emerged as a deployment-friendly recipe~\citep{li2024snapkv,cai2024pyramidkv}: it converts the runtime memory wall into a single fixed cost with zero decoding-time intervention. This regime exposes a \emph{statistical mismatch}, also suggested by prior observations of retrieval-head specialization and long-context failure modes~\citep{fu2024not,hsieh2024ruler}. The signals needed to fix an irreversible retention decision live at the cross-prompt scale, while prior online selectors estimate them at the within-prompt scale. We make this mismatch explicit and show that fixing it, without changing what is selected but only \emph{where the policy is computed}, closes most of the gap to FullKV at a fixed 512-token budget.

\textbf{The hard core of the problem.} The defining constraint of prefill-only compression is that the retention decision is \emph{irreversible}: the signals that ultimately reveal which tokens matter, namely per-head retrieval roles and prompt-level compression sensitivity, only emerge during decoding, but no token can be brought back. A one-shot decision under a tight budget therefore leaves no room for error correction, and any selection error becomes permanent information loss.

\textbf{Three error sources, jointly.} We identify three structural error sources that any one-shot decision must control simultaneously, since failing on any single one is enough to trigger tail failure:
\begin{itemize}\setlength\itemsep{0.15em}
    \item \textbf{(E1) Single-prompt statistical noise.} Per-prompt attention is heavy-tailed and scale-biased; large magnitudes do not imply semantic importance, and small estimation errors shift the selection boundary irreversibly.
    \item \textbf{(E2) Head heterogeneity.} Attention heads are functionally specialized, with retrieval-critical heads coexisting with consistently noisy ones~\citep{voita2019analyzing, fu2024not}, so naïvely aggregating across heads lets noisy heads dominate the score.
    \item \textbf{(E3) Prompt-level risk variance.} Different prompts compress to different limits; a one-size-fits-all retention rule under-allocates to high-entropy or high-PPL prompts and triggers catastrophic loss precisely on the hardest examples.
\end{itemize}
Representative prefill-only compressors---SnapKV~\citep{li2024snapkv}, PyramidKV~\citep{cai2024pyramidkv}, DynamicKV~\citep{zhou2024dynamickv}, AdaKV~\citep{feng2024ada}, HeadKV~\citep{fu2024not}, CAKE~\citep{qin2024cake}, and ChunkKV~\citep{liu2025chunkkv}---chase E1--E3 with the same recipe: re-estimate the corrective signals \emph{online}, on the single current prompt, from instantaneous attention statistics. The granularities differ (token, chunk, layer, head, task), but the statistical unit is the same, namely \emph{one prompt}, and Figure~\ref{fig:e1e2e3_diagnostic} shows that this is the wrong unit for the correction. Per-prompt attention is heavy-tailed (E1) and has high cross-head variance (E2), but the \emph{ranking} of which heads are reliable, and the \emph{mapping} from risk features to a safe threshold (E3), are stable at the cross-prompt scale. Single-prompt estimators cannot recover them no matter how cleverly the budget is granularized. We call the resulting system \textsc{CompilerKV}: it compiles stable model-level priors into a deterministic prefill-retention policy before inference, leaving only the prompt-specific residual (token utility and a risk-bin lookup) to the online pass.

\begin{figure}[!t]
\centering
\includegraphics[width=0.92\textwidth]{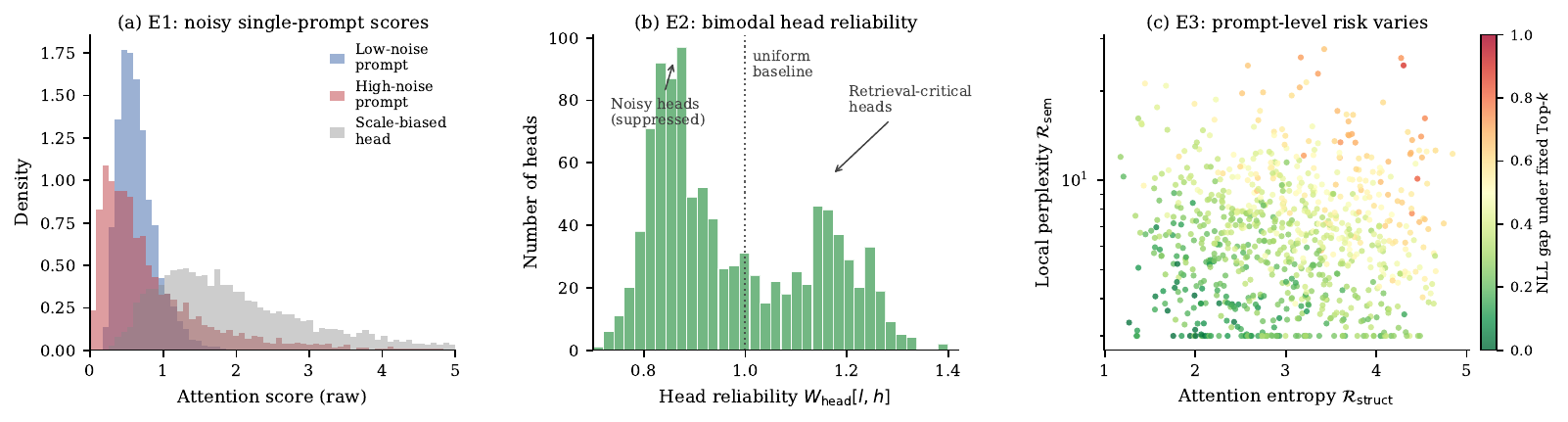}
\caption{\textbf{Empirical evidence that E1--E3 are real, distinct failure modes on calibration data.}
\emph{(a)} The same token's attention score has very different distributions across prompts (low-noise vs.\ high-noise prompt) and across heads (a scale-biased head spans a much wider range), confirming E1.
\emph{(b)} The aggregated distribution of compiled head-reliability weights is bimodal: a substantial cluster of heads is consistently down-weighted (noise emitters), while a smaller cluster receives weights $\!>\!1.1$ (retrieval-critical), confirming E2.
\emph{(c)} The compression damage of a fixed Top-$k$ rule (color: NLL gap to FullKV) varies systematically with attention entropy and local perplexity, confirming E3 as a prompt-level effect that no global $k$ can absorb.}
\label{fig:e1e2e3_diagnostic}
\end{figure}

\textbf{Operational view.}
Online heuristics re-estimate importance from each prompt's own attention statistics.
DuoAttention-style offline head typing makes a one-time binary assignment but has no
prompt-risk control at inference. \textsc{CompilerKV} answers a stricter question:
given compiled priors on head reliability and prompt risk, what threshold makes the
one-shot retention decision safe?

\textbf{Our reframing: the corrective signals are stable architectural priors.} Functional specialization of attention heads is largely fixed at pre-training time~\citep{voita2019analyzing}; retrieval-head studies further show stable head roles across inputs~\citep{fu2024not}. We confirm empirically that head-reliability rankings transfer across disjoint calibration corpora \emph{for every evaluated backbone}: LLaMA-3, Mistral, Qwen2, and InternLM obtain aggregate Spearman correlations in the $0.89$--$0.91$ range, with mean $\bar{\rho}=0.90$ between arXiv$+$PubMed and ShareGPT$+$UltraChat (\S\ref{sec:cross_domain}). Likewise, the mapping from prompt-level risk (attention entropy, local perplexity) to a safe retention threshold is a property of the \emph{calibration distribution}, not any individual prompt. This stability is the missing piece: it means E1--E3 can be solved \emph{once, offline}, on a held-out calibration corpus, and the result reused for every future prompt via $O(1)$ table lookups.

\textbf{Compile once, deploy across models.} This turns head reliability from a tuning trick into a reusable scientific object. Appendix~\ref{sec:model_transfer} reports direct model-to-model table transfer: same-family reuse costs $0.3$--$0.4$ LongBench points, while broader source-to-target transfer costs only $0.4$--$0.8$ points relative to target-model compilation. This portability is where continuous reliability weights matter: unlike binary retrieval/streaming head typing, they preserve a graded ordering that can be depth-aligned and reused across model variants.

\textbf{\textsc{CompilerKV} as a compiled retention policy.} We instantiate this idea as an offline-compiled KV-retention policy whose components map one-to-one to E1--E3 by construction. The novelty is not assigning heads different weights, which several prior methods also do; it is the \emph{coupled head--risk compilation}: continuous head reliability changes the utility distribution, while the risk gate chooses a safe threshold for that distribution under a hard budget.
\begin{itemize}\setlength\itemsep{0.15em}
    \item \emph{Stabilized utility} (E1): a window-aggregated, head-relative score that removes the two largest sources of online noise, transient attention spikes and inter-head scale bias, without any learning.
    \item \emph{Head Heterogeneity Table} (E2): per-head reliability weights compiled offline via conservative Q-learning (CQL), giving reliable heads ``veto power'' so that a single retrieval-critical head can save a token from noisy-head consensus.
    \item \emph{Risk-Adaptive Threshold Gate} (E3): a small lookup table mapping (entropy-bin, PPL-bin) to a retention threshold, compiled with the same offline CQL estimator. Conservative regularization explicitly biases the table toward safer retention in rare high-risk states, the exact regime where online heuristics fail.
\end{itemize}
The three components are not independent additions: they form one irreversible-decision policy. The runtime tables are factorized for efficiency, but they are calibrated against the same downstream compressed-cache reward and final retention operator. Proposition~\ref{thm:bound} provides a formal anchor: deterministic approximation error is controlled by discarded tail mass, while finite-table calibration error is controlled at the prompt level rather than by assuming independent attention rows.

\section{Related Work}

\textbf{Prefill-only KV cache selection.} The dominant approach to prefill-only compression selects a token subset at the end of prefill based on instantaneous attention statistics and freezes the result for all decoding steps. H2O~\citep{zhang2023h2o} retains tokens with the highest accumulated attention scores; SnapKV~\citep{li2024snapkv} uses a query observation window at the prompt tail; VATP~\citep{guo2024attention} adds value-norm magnitude; ChunkKV~\citep{liu2025chunkkv} extends selection to contiguous semantic chunks. Layer-level methods (PyramidKV~\citep{cai2024pyramidkv}) assign larger budgets to lower layers; head-level methods (AdaKV~\citep{feng2024ada}, HeadKV~\citep{fu2024not}, CAKE~\citep{qin2024cake}, ZigZagKV~\citep{zhong2024zigzagkv}) further redistribute budget across attention heads per prompt; task-conditioned methods (DynamicKV~\citep{zhou2024dynamickv}) shape the budget based on inferred task type. Despite varying granularity, all share the same signal source: online, single-prompt attention statistics, which Figure~\ref{fig:e1e2e3_diagnostic} diagnoses as structurally mismatched to the three error sources (E1--E3) that determine prefill compression quality. \textsc{CompilerKV} jointly compiles continuous head reliability and prompt-risk thresholds offline, combining them with prompt-specific utility and risk coordinates through $O(1)$ table lookups.

\textbf{Decode-time eviction and sparse attention.} A complementary line of work~\citep{liu2023deja, xiao2023efficient, tang2024quest, lv2025kvpruner} retains the ability to evict tokens or reweight attention \emph{during} decoding, using generation-time feedback to correct earlier decisions. These methods tolerate online noise because they can adapt; our setting is stricter: retention is fixed at prefill, and no correction is possible. KVzip~\citep{kim2025kvzip} occupies a middle ground by using the LLM itself to score KV pair importance through context reconstruction, producing a query-agnostic importance estimate that is more accurate than single-forward-pass attention but requires an extra LLM forward pass per context. \textsc{CompilerKV}'s compilation phase amortizes a similar accuracy benefit across all future prompts rather than paying it per context. Sparse-attention methods (e.g., Quest~\citep{tang2024quest}, SparQ~\citep{ribar2024sparq}) reduce attention FLOPS at decode time without modifying the stored cache; they are orthogonal to prefill-only compression and can be composed with \textsc{CompilerKV}.

\textbf{Head heterogeneity and architectural priors.} \citet{voita2019analyzing} established that attention heads are functionally specialized at pre-training time, with most heads performing redundant or noisy operations. \citet{fu2024not} characterized retrieval heads, a small subset that drives long-context recall, and showed their identity is stable across input instances. \citet{du2025heads} further documented how single-prompt attention statistics are heavy-tailed and scale-biased across heads, directly motivating our E1/E2 diagnosis. 
AdaKV \citep{feng2024ada} and CAKE \citep{qin2024cake} incorporate prompt-dependent head-level signals, while HeadKV \citep{fu2024not} uses calibrated head importance as an allocation prior.
% Prior prefill-only methods that incorporate head-level signals (AdaKV, HeadKV, CAKE) treat per-head importance as a \emph{per-prompt} quantity to be re-estimated at each inference call. 
\textsc{CompilerKV} instead treats head reliability as a \emph{model-level} prior, compiled once from calibration data, an observation enabled by the cross-corpus stability we establish across all four backbones in \S\ref{sec:cross_domain}.

\textbf{Offline and data-driven compression.} Several works use offline data to inform compression decisions. DuoAttention~\citep{xiao2024duoattention} is the closest conceptual predecessor: it identifies retrieval vs.\ streaming heads offline and assigns different cache strategies. This is valuable, but it is a \emph{binary head-type compiler}; once the assignment is made, the runtime policy has no prompt-risk control and no continuous interaction between head reliability and the token-selection boundary. \textsc{CompilerKV} compiles a different object, a \emph{risk-conditioned retention policy}: its head table is continuous, the gate is prompt-risk conditioned, and the two are optimized against the same prefill-only reward, so the threshold is learned for the utility distribution induced by the head table. MagicPIG~\citep{chen2024magicpig} and ShadowKV~\citep{sun2024shadowkv} also use offline-computed statistics, but they accelerate decode-time retrieval rather than solving irreversible prefill-only retention. Profiling-based adaptive-cache approaches such as FastGen~\citep{ge2023context} use model-specific statistics to choose cache strategies; \textsc{CompilerKV} keeps the base LLM frozen and compiles lightweight retention tables only. The novelty is therefore not ``head awareness'' alone, but joint offline compilation of continuous head reliability and prompt-risk thresholds for fixed-cache prefill compression.

\section{Preliminaries}
\label{sec:prelim}

\textbf{KV cache and prefill-only compression.}
For a Transformer with $L$ layers and $H$ heads~\citep{vaswani2017attention},
a prompt of length $T$ produces at layer $l$ a key--value cache
$\{K^{(l)}, V^{(l)}\}$ of size $O(TH_\text{kv} d_{\text{head}})$.
At decoding step $t$, the per-head attention weight
$A_{j,t}^{(l,h)} \propto \exp\!\bigl((q_t^{(l,h)})^\top k_j^{(l,h)}/\sqrt{d_{\text{head}}}\bigr)$
draws on all cached positions.
\emph{Prefill-only compression} commits once, at the end of prefill, to an index
set $\mathcal{S}^{(l)}\!\subset\!\{1,\dots,T\}$ with $|\mathcal{S}^{(l)}|\le B_l$,
retaining $\tilde{K}^{(l)}\!=\!K^{(l)}[\mathcal{S}^{(l)}]$,
$\tilde{V}^{(l)}\!=\!V^{(l)}[\mathcal{S}^{(l)}]$ for the entire decoding pass.
The irreversibility of $\mathcal{S}^{(l)}$ is the core difficulty: selection errors
are permanent.

For grouped-query attention (GQA), $H=H_q$ denotes query heads, $H_{\mathrm{kv}}$ denotes physical KV heads, and $g(h)$ maps each query head to its KV head. The notation $v^{(l,h)}$ below abbreviates $v^{(l,g(h))}$; query heads in the same group share value-norm statistics.

\textbf{Offline RL via Conservative Q-Learning (CQL).}
We cast compression policy learning as a single-step contextual
bandit~\citep{lattimore2020bandit}: context $s\in\mathcal{S}$ indexes a head or risk
bin, action $a\in\mathcal{A}$ encodes a compression choice, and reward $r$ measures
fidelity on a calibration prompt.
The horizon-1 formulation is exact: prefill-only compression admits no future state
to bootstrap from, so CQL acts solely as a \emph{support-regularized offline
estimator}---penalizing Q-values for rarely sampled actions---rather than a
multi-step planner.
Given dataset $\mathcal{D}=\{(x_i,s_i,a_i,r_i)\}$ and behavior policy
$\pi_{\mathrm{beh}}$, we minimize
\begin{equation}\label{eq:cql}
\min_\theta\;
\alpha\!\left(
  \mathbb{E}_{s\sim\mathcal{D}}\!\left[\log\!\sum_{a}\exp Q_\theta(s,a)\right]
  -\mathbb{E}_{(s,a)\sim\pi_{\mathrm{beh}}}\!\left[Q_\theta(s,a)\right]
\right)
+\frac{1}{2}\,\mathbb{E}_{(s,a,r)\sim\mathcal{D}}\!\left[(Q_\theta(s,a)-r)^2\right].
\end{equation}
The conservative term replaces $\arg\max_a\hat{r}(s,a)$ with a penalized rule,
preventing aggressive thresholds on rare high-entropy or high-PPL prompts.
We apply Eq.~\eqref{eq:cql} uniformly to both compiled tables
(\S\ref{sec:head_table}, \S\ref{sec:gate}), with reward
\begin{equation}\label{eq:reward}
r(x,s,a)
= -\bigl(\mathcal{L}_{\text{comp}}(x;s,a)-\mathcal{L}_{\text{full}}(x)\bigr)
  -\lambda\,\Psi(x,s,a),
\end{equation}
where $\mathcal{L}_{\text{comp}}$, $\mathcal{L}_{\text{full}}$ are NLL under
compressed and full caches, and
$\Psi\!=\!\max(0,|\mathcal{I}_{\text{cand}}|-B_l)$ penalizes over-retention.
We set $\alpha\!=\!0.75$; the risk gate uses $\lambda\!=\!\beta_{\mathrm{bud}}/T$
($\beta_{\mathrm{bud}}\!=\!1.0$) and the head table uses $\lambda\!=\!0$.
Full details are in Appendix~\ref{sec:offline_table_compilation}.

\begin{figure*}[!ht]
    \centering
    \includegraphics[width=1\linewidth]{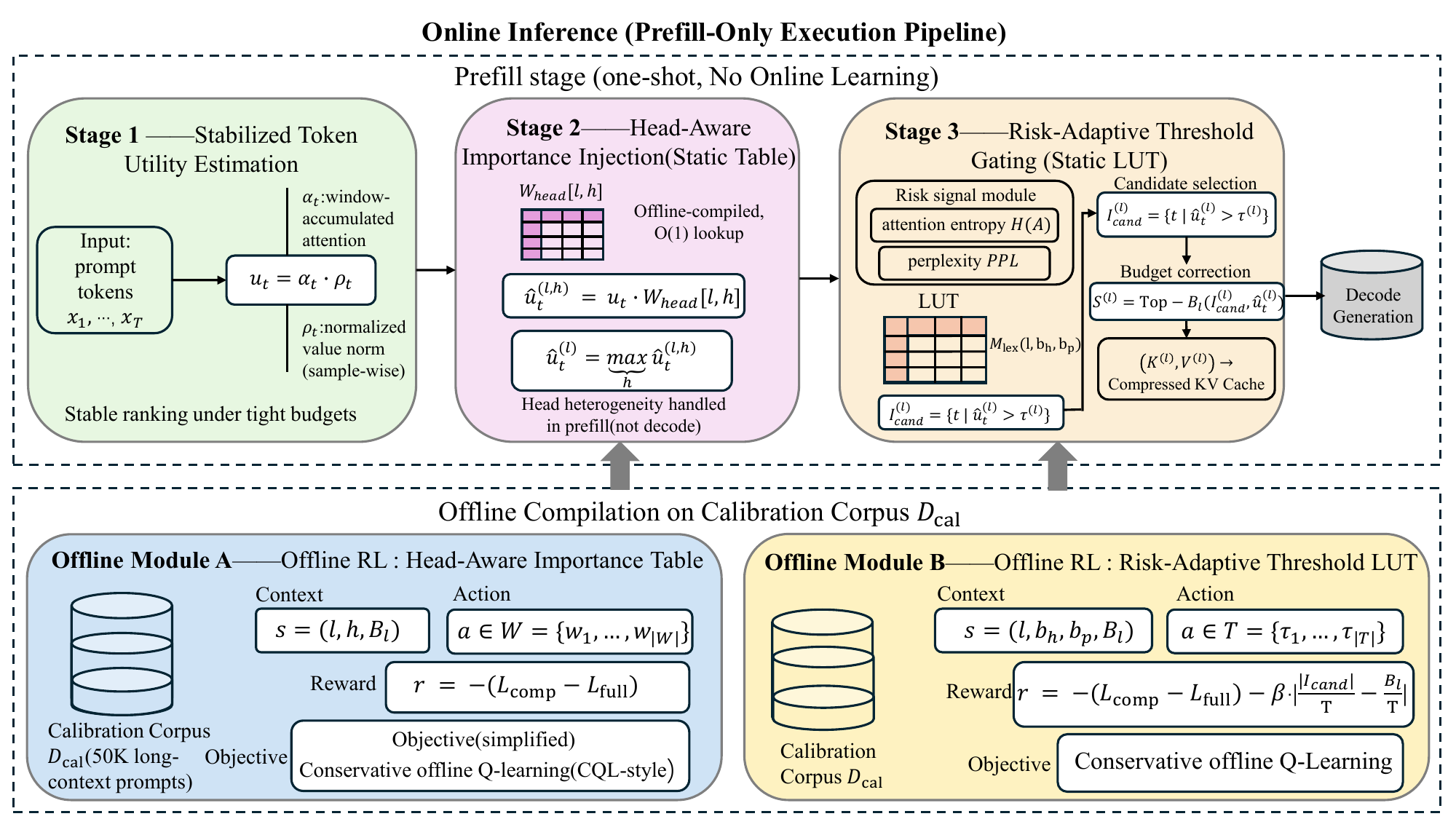} 
    \caption{\textbf{Overview of the \textsc{CompilerKV} framework.} The pipeline has three stages: (1) a stabilized utility score that suppresses transient noise; (2) a compiled Head Heterogeneity Table that governs functional differences among attention heads; and (3) a compiled Risk-Adaptive Gating Table that calibrates the retention threshold to each prompt's complexity.}
    \label{fig:framework}
\end{figure*}

\section{\textsc{CompilerKV}}

Figure~\ref{fig:framework} illustrates the overall \textsc{CompilerKV} framework, 
comprising an online prefill-only execution path and two offline compilation modules. 
To address E1--E3, \textsc{CompilerKV} employs three components within a single 
prefill pass: a parameter-free \emph{stabilized utility} (E1), an offline-compiled 
\emph{Head Heterogeneity Table} (E2), and an offline-compiled \emph{Risk-Adaptive 
Threshold Gate} (E3). These components operate sequentially---ranking tokens, 
rescaling by head reliability, and adjusting the budget threshold for risky 
prompts---after which the KV cache is frozen for all subsequent decoding. Online 
overhead is minimal: two $O(1)$ table lookups per layer and a single sort over 
the candidate pool, with no decoder-side intervention or cache rewriting.

\textbf{Factorized compiler, coupled reward evaluation.} The runtime representation is deliberately factorized: one table indexes head coordinates $(l,h,B_l)$ and the other indexes risk coordinates $(l,b_\text{ent},b_\text{ppl},B_l)$. We do \emph{not} learn a dense joint action over all heads and thresholds. Instead, each table is fitted with the other side of the final retention policy active, so the measured reward includes the score distribution induced by the current head weights and the token set induced by the threshold gate. The interaction therefore enters through reward evaluation and the final selection operator, while inference remains two lookup tables. Table~\ref{tab:ablation_and_compilation} includes an independent-compilation probe that verifies calibrating the two tables against stale boundaries is weaker than calibrating them under the paired retention policy.

\textbf{Inference-time complexity.} Let $T$ be sequence length, $W=|\Omega|$ the observation window, and $L,H$ the number of layers and heads. \textsc{CompilerKV} still reads the observation-window attention mass, so the prefill-side scan is
\begin{equation}
\label{eq:runtime_complexity}
O(LHTW),
\end{equation}
the same as window-based prefill compressors such as SnapKV~\citep{li2024snapkv}. Our savings come after this scan: the averaged window mass is reused across layers and heads, both compiled tables are $O(1)$ lookups, and the risk gate shrinks each layer's candidate pool before Top-$B_l$ selection. All post-scan operations are at most $O(LT\log B)$ and absorbed into the leading term, with no decoder-time scoring or cache rewriting.

\subsection{Stabilized Token Utility (Targets E1)}
\label{sec:utility}

Online attention statistics suffer from two coupled noise sources: transient attention spikes inside a single prompt, and inter-head scale bias that biases aggregate scores toward heads with naturally larger magnitudes. Our utility removes both \emph{without learning}, treating E1 as a clean signal-processing fix:
\begin{equation}
\label{eq:stabilized_utility}
u_t^{(l,h)} = \alpha_t \cdot \rho_t^{(l,h)},\quad
\alpha_t = \frac{T}{W}\!\!\sum_{j\in\Omega} \bar{A}_{j,t},\quad
\rho_t^{(l,h)} = \frac{\|v_t^{(l,h)}\|_2}{\tfrac{1}{T}\sum_k\|v_k^{(l,h)}\|_2},
\end{equation}
where $\Omega = \{T{-}W{+}1,\dots,T\}$, $W=|\Omega|$, and $\bar{A}_{j,t} = \tfrac{1}{LH}\sum_{l,h} A_{j,t}^{(l,h)}$. The factor $T/W$ makes $\alpha_t$ have mean scale near one because $\sum_t\sum_{j\in\Omega}\bar A_{j,t}=W$, which is why absolute thresholds in $[0.8,1.0]$ remain meaningful even for long contexts. Window aggregation smooths transient spikes; head-relative normalization removes scale bias.

\subsection{Head Heterogeneity Table (Targets E2)}
\label{sec:head_table}

E2 cannot be fixed by signal processing because it is an architectural fact: heads differ in functional reliability, and which heads are reliable is determined at pre-training, not at inference. Treating heads equally, as $u_t^{(l,h)}$ alone implicitly does, lets noisy heads pollute the aggregate. We therefore compile reliability \emph{offline} into a frozen table $\mathbf{W}_{\text{head}}\!\in\!\mathbb{R}^{L\times H}$ that the inference pass consults via $O(1)$ lookup.

\textbf{Compilation.} We instantiate the contextual bandit of \S\ref{sec:prelim} with state $s_{l,h}=(l,h,B_l)$, action $a_{l,h}\in\mathcal{W}$ a discretized reliability scalar, and reward Eq.~\eqref{eq:reward} with $\lambda=0$ (head weighting does not directly violate budget). Each head's marginal contribution is measured by perturbing its $a_{l,h}$ alone, but the reward is always evaluated under the \emph{full retention pipeline}, including the current risk gate $\mathbf{T}_\text{gate}$ and the budget-clamp operator (Eq.~\eqref{eq:budget_correction}); the two tables are alternated to convergence. The interaction between heads and threshold therefore enters through the reward landscape rather than through a joint action space. Solving Eq.~\eqref{eq:cql} yields $Q_\theta$, from which we extract:
\begin{equation}
\label{eq:head_compilation}
\mathbf{W}_{\text{head}}[l,h] = \operatorname*{argmax}_{w \in \mathcal{W}} Q_{\theta}(s_{l,h}, w).
\end{equation}

\textbf{Inference: weighted max-pooling gives reliable heads veto power.} The per-token aggregate is
\begin{equation}
\label{eq:max_pooling}
\hat{u}_{t}^{(l)} = \max_{h} \bigl( u_{t}^{(l,h)} \cdot \mathbf{W}_{\text{head}}[l,h] \bigr).
\end{equation}
Max-pooling preserves the strongest reliability-weighted contribution instead of averaging it with weaker scores. The attention summary in Eq.~\eqref{eq:stabilized_utility} is averaged across layers and heads. A single retained set is gathered for all KV heads within each layer.

\subsection{Risk-Adaptive Threshold Gate (Targets E3)}
\label{sec:gate}

Even with E1 and E2 controlled, applying a fixed Top-$k$ rule across prompts re-opens E3: high-entropy or high-PPL prompts need conservative retention, low-risk prompts tolerate aggressive pruning, and a one-size-fits-all $k$ produces tail failures on the high-risk extreme. We compile the risk-to-threshold mapping offline using the common offline CQL estimator.

\textbf{Two complementary risk signals.} Within $\Omega$, structural risk is captured by attention entropy and semantic risk by local perplexity:
\begin{equation}
\label{eq:risk_features}
\mathcal{R}_\text{struct} = -\!\sum_{t=1}^{T}\!\bar{A}^{'}_t \log\bar{A}^{'}_t,\quad
\mathcal{R}_\text{sem} = \exp\!\Bigl(-\tfrac{1}{|\Omega|}\!\!\sum_{j\in\Omega}\!\!\log P(x_j|x_{<j})\Bigr),
\end{equation}
where $\bar{A}^{'}_t = \tfrac{1}{|\Omega|LH}\sum_{j\in\Omega,l,h} A_{j,t}^{(l,h)}$. High entropy means no clear focal point (pruning is risky); a perplexity spike indicates complex semantic dependencies (history is critical). Together these signals cover structural and semantic risk; our error analysis (\S\ref{sec:cross_domain}) finds that fewer than $8\%$ of failure cases lie outside this signal pair.

\textbf{Compilation.}
We discretize $(\mathcal{R}_{\text{struct}},\mathcal{R}_{\text{sem}})$ into bins
$(b_{\text{ent}},b_{\text{ppl}})$ and instantiate the bandit with state
$s_l=(l,b_{\text{ent}},b_{\text{ppl}},B_l)$, action $a\in[0.8,1.0]$, and reward
Eq.~\eqref{eq:reward} with $\lambda=\beta_{\mathrm{bud}}/T$,
$\beta_{\mathrm{bud}}=1.0$.
The hinge penalty $\Psi$ is asymmetric by design: only over-retention is penalized,
since under-retention is absorbed by the elastic pool
(Eq.~\eqref{eq:budget_correction}).
This asymmetry, combined with the conservative regularization in
Eq.~\eqref{eq:cql}, produces the key E3 fix: the compiled table raises thresholds
on low-risk prompts and lowers them on high-risk ones, biasing toward safer retention
in the rare high-entropy or high-PPL states where online heuristics fail.
Solving Eq.~\eqref{eq:cql} yields $\mathbf{T}_{\text{gate}}$, returning
$\tau^{(l)}=\mathbf{T}_{\text{gate}}[s_l]$ at inference.
We use a $20{\times}4$ entropy$\times$PPL grid with $\tau\!\in\![0.8,1.0]$;
the grid design follows a bias--variance argument in
Appendix~\ref{sec:offline_table_compilation}.

\textbf{Final selection: elastic candidate pool with budget clamp.} Define $\mathcal{I}_\text{cand}^{(l)} = \{t \mid \hat{u}_t^{(l)} \ge \tau^{(l)}\}$. The retained set
\begin{equation}
\label{eq:budget_correction}
\mathcal{S}^{(l)} =
\begin{cases}
\mathcal{I}_\text{cand}^{(l)}, & |\mathcal{I}_\text{cand}^{(l)}| \le B_l, \\
\text{Top-}B_l(\mathcal{I}_\text{cand}^{(l)}, \hat{u}^{(l)}), & |\mathcal{I}_\text{cand}^{(l)}| > B_l,
\end{cases}
\end{equation}
permits elastic under-retention on low-risk prompts and clamps to the budget on high-risk ones. Our ablation shows this elasticity is the single most critical design choice ($\sim$$4.2$-pt drop when removed; Table~\ref{tab:ablation_and_compilation}). The compressed cache $(\tilde K^{(l)}, \tilde V^{(l)})$ is then frozen for the entire decoding pass.

\paragraph{Theoretical grounding.}
\begin{proposition}[Tail error and finite-table calibration]
\label{thm:bound}
Fix a layer $l$ and head $h$. For each decoding query $t$, let $p_t=A_{t,:}^{(l,h)}$ be the full attention row and let $\mathcal{S}^{(l)}$ be the retained prefill index set. Define the discarded mass
$m_t^{(l,h)}=p_t((\mathcal{S}^{(l)})^c)$ and assume $m_t^{(l,h)}<1$ so that the compressed row is the renormalized restriction
$\tilde p_t(j)=p_t(j)\mathbf{1}\{j\in\mathcal{S}^{(l)}\}/(1-m_t^{(l,h)})$.
With $\epsilon_{\rm tail}^{(l,h)}=(T^{-1}\sum_t(m_t^{(l,h)})^2)^{1/2}$,
\begin{equation}
\label{eq:tail_bound_main}
\frac{1}{H}\sum_{h=1}^{H}\|A^{(l,h)}-\tilde A^{(l,h)}\|_F
\le 2\sqrt{T}\,\frac{1}{H}\sum_{h=1}^{H}\epsilon_{\rm tail}^{(l,h)} .
\end{equation}
Moreover, let $\mathcal{A}$ be the finite action table used by the offline compiler. For $N$ independent calibration prompts, suppose the per-prompt loss $\ell(p,a)$ lies in an interval of width $\Delta_R=R_{\max}-R_{\min}$, and define $R(a)=\mathbb{E}[\ell(p,a)]$ and $\hat R(a)=N^{-1}\sum_{i=1}^N\ell(p_i,a)$. Then, with probability at least $1-\delta$,
\begin{equation}
\label{eq:calibration_uniform_main}
\sup_{a\in\mathcal{A}}|R(a)-\hat R(a)|
\le \Delta_R\sqrt{\frac{\log(2|\mathcal{A}|/\delta)}{2N}} .
\end{equation}
Consequently, if $\hat a\in\arg\min_{a\in\mathcal{A}}\hat R(a)$ and $a^\star\in\arg\min_{a\in\mathcal{A}}R(a)$, then
\begin{equation}
\label{eq:calibration_excess_main}
R(\hat a)-R(a^\star)
\le 2\Delta_R\sqrt{\frac{\log(2|\mathcal{A}|/\delta)}{2N}} .
\end{equation}
\end{proposition}
Eq.~\eqref{eq:tail_bound_main} says that the approximation error is controlled by the discarded tail mass targeted by stabilized utility and reliable-head vetoing. Eqs.~\eqref{eq:calibration_uniform_main}--\eqref{eq:calibration_excess_main} provide the statistical reason for compilation: the table concentrates over independent \emph{prompts}, not over attention rows inside a prompt. \textsc{CompilerKV} therefore does not require a row-wise independence assumption; it lowers variance by moving the corrective decision to the calibration scale.

% \begin{figure*}[!ht]
%     \centering
%     % 1. FullKV (Top - Reference)
%     \begin{subfigure}[b]{0.95\linewidth}
%         \centering
%         \includegraphics[width=\linewidth]{mistral-7b-instruct-v0.2_fullkv.pdf}
%         \caption{FullKV (Avg. Score: 0.92) - Lossless Upper Bound}
%         \label{fig:niah_fullkv}
%     \end{subfigure}
    
%     \vspace{0.2cm} % 垂直间距
    
%     % 2. DynamicKV (Strong Baseline)
%     \begin{subfigure}[b]{0.95\linewidth}
%         \centering
%         \includegraphics[width=\linewidth]{mistral-7b-instruct-v0.2_dynamickv.pdf}
%         \caption{DynamicKV (Avg. Score: 0.83) - SOTA Baseline}
%         \label{fig:niah_dynamic}
%     \end{subfigure}
    
%     \vspace{0.2cm}
    
%     % 3. CompilerKV (Ours - Bottom)
%     \begin{subfigure}[b]{0.95\linewidth}
%         \centering
%         \includegraphics[width=\linewidth]{mistral-7b-instruct-v0.2_compilerkv.pdf}
%         \caption{\textbf{CompilerKV (Avg. Score: 0.86)} - Ours}
%         \label{fig:niah_compilerkv}
%     \end{subfigure}
    
%     \caption{\textbf{Needle-in-a-Haystack Pressure Test Comparison.} 
%     We visualize the retrieval accuracy (Green=100\%, Blue=0\%) on Mistral-7B across varying context lengths (x-axis) and needle depths (y-axis). 
%     While the strong baseline \textbf{DynamicKV} (b) exhibits fragmentation and information loss at extreme lengths, \textbf{CompilerKV} (c) maintains a robust, dense retrieval pattern that closely mirrors the lossless \textbf{FullKV} oracle (a), demonstrating superior stability under compression.}
%     \label{fig:niah_comparison}
% \end{figure*}

\begin{table}[!ht]
\centering
\caption{\textbf{LongBench results (512-token budget).} The best compressed result is \textbf{bold} and the second best is \underline{underlined}; FullKV is shown as the lossless reference and is not used for compressed-method ranking. AdaKV and CAKE are included as recent head- and layer-adaptive baselines.}
\resizebox{\textwidth}{!}{
\scriptsize
\renewcommand{\arraystretch}{0.92}
\setlength{\tabcolsep}{2.3pt}
\begin{tabular}{c|c|ccc|ccc|ccc|ccc|cccc|c}
\toprule
& & \multicolumn{3}{c|}{\textit{Single-Doc QA}} & \multicolumn{3}{c|}{\textit{Multi-Doc QA}} & \multicolumn{3}{c|}{\textit{Summarization}} & \multicolumn{3}{c|}{\textit{Few-Shot Learning}} & \multicolumn{4}{c|}{\textit{Synthetic \& Code}} & \\
\cmidrule(lr){3-5} \cmidrule(lr){6-8} \cmidrule(lr){9-11} \cmidrule(lr){12-14} \cmidrule(lr){15-18}
\raisebox{3ex}{\textbf{Model}} &
\raisebox{3ex}{\textbf{Method}} & 
\rotatebox{45}{\textbf{NarrativeQA}} & \rotatebox{45}{\textbf{Qasper}} & \rotatebox{45}{\textbf{MF-en}} & 
\rotatebox{45}{\textbf{HotpotQA}} & \rotatebox{45}{\textbf{2WikiMQA}} & \rotatebox{45}{\textbf{MuSiQue}} & 
\rotatebox{45}{\textbf{GovReport}} & \rotatebox{45}{\textbf{QMSum}} & \rotatebox{45}{\textbf{MultiNews}} & 
\rotatebox{45}{\textbf{TREC}} & \rotatebox{45}{\textbf{TriviaQA}} & \rotatebox{45}{\textbf{SAMSum}} & 
\rotatebox{45}{\textbf{Lcc}} & \rotatebox{45}{\textbf{RB-P}} & \rotatebox{45}{\textbf{Pcount}} & \rotatebox{45}{\textbf{PRe}} & 
\raisebox{3ex}{\textbf{\textit{Avg.}}} \\
\midrule
\multirow{9}{*}{\rotatebox{90}{\parbox{2cm}{\centering \textbf{InternLM-2.5}\\ \textbf{7B-Chat-1M}}}}
& FullKV & 22.47 & 27.58 & 39.98 & 40.96 & 33.52 & 26.61 & 33.01 & 25.18 & 26.28 & 72.5 & 86.76 & 39.84 & 55.86 & 57.90 & 2.92 & 100.00 & 43.21 \\
\cmidrule{2-19}
& SnapKV & 16.86 & 23.28 & 36.24 & 32.14 & 19.89 & 23.21 & 17.81 & 23.18 & 22.44 & 71.0 & 84.05 & 34.34 & 50.32 & \underline{53.34} & 1.00 & \textbf{96.50} & 37.85 \\
& PyramidKV & 17.62 & 21.08 & 37.52 & 32.21 & 21.31 & 22.03 & 19.37 & \underline{24.06} & 22.22 & 73.0 & 83.94 & 34.61 & 50.45 & 49.72 & 1.05 & 95.50 & 37.86 \\
& DynamicKV & 17.77 & 23.87 & 37.74 & 32.98 & 21.13 & 23.21 & 19.13 & 23.49 & 22.48 & \underline{75.0} & 84.89 & 36.70 & 50.70 & 51.08 & 0.91 & 95.50 & 38.54 \\
& AdaKV & 18.23 & 24.64 & 38.30 & 33.64 & 21.66 & \underline{23.55} & 19.74 & 23.96 & 22.92 & \textbf{75.3} & \underline{85.28} & 37.23 & 51.31 & 51.91 & \underline{1.30} & \underline{95.78} & 39.05 \\
& HeadKV & 17.97 & 24.67 & 38.24 & 33.48 & 21.53 & 23.31 & 19.73 & 23.89 & 22.78 & \underline{75.0} & 85.09 & 37.10 & 51.20 & 51.88 & 1.01 & 95.50 & 38.90 \\
& CAKE & 18.32 & 25.04 & 38.49 & 33.78 & 21.81 & 23.44 & 20.18 & 24.02 & 23.15 & 75.0 & 85.12 & 37.30 & 51.42 & 52.14 & 1.16 & 95.62 & 39.12 \\
& ChunkKV & \underline{18.47} & \underline{25.17} & \underline{38.64} & \underline{33.88} & \underline{21.93} & 23.51 & \underline{20.33} & \textbf{24.29} & \underline{23.28} & \underline{75.0} & 85.19 & \underline{37.40} & \underline{51.50} & 52.18 & 1.21 & 95.50 & \underline{39.22} \\
\rowcolor{gray!10} & \textbf{Ours} & \textbf{22.87} & \textbf{27.22} & \textbf{39.37} & \textbf{40.49} & \textbf{33.25} & \textbf{26.05} & \textbf{33.41} & 24.05 & \textbf{23.68} & \underline{75.0} & \textbf{86.64} & \textbf{39.11} & \textbf{55.16} & \textbf{56.88} & \textbf{3.00} & 95.50 & \textbf{42.61} \\
\midrule

\multirow{9}{*}{\rotatebox{90}{\parbox{2cm}{\centering \textbf{LLaMA-3}\\\textbf{8B-Instruct}}}}
& FullKV & 25.16 & 31.81 & 39.59 & 43.09 & 36.15 & 21.77 & 28.62 & 23.34 & 26.33 & 75.0 & 90.50 & 42.36 & 59.04 & 53.93 & 5.20 & 69.25 & 41.95 \\
\cmidrule{2-19}
& SnapKV & 24.62 & 22.78 & 37.88 & 42.96 & 34.82 & 20.25 & 22.63 & 22.54 & 23.93 & 70.0 & 90.39 & 40.30 & 60.27 & 55.85 & 5.74 & 69.50 & 40.28 \\
& PyramidKV & 24.48 & 23.51 & 36.24 & 42.33 & 31.95 & 20.73 & 23.37 & \underline{23.01} & 24.37 & 72.5 & 90.43 & 40.54 & 59.27 & 54.87 & 5.88 & 69.50 & 40.19 \\
& DynamicKV & 24.80 & 24.62 & 36.69 & 44.13 & 33.25 & 20.82 & 23.00 & 22.54 & 24.12 & 72.5 & 90.39 & 40.76 & 61.21 & 56.91 & 5.78 & 69.50 & 40.69 \\
& AdaKV & 25.06 & 27.71 & 37.64 & 44.01 & 35.25 & 21.50 & 23.49 & 22.90 & 24.64 & 73.1 & \underline{90.64} & 41.25 & 61.63 & 57.61 & 6.17 & \underline{70.35} & 41.43 \\
& HeadKV & 24.75 & \underline{29.75} & 38.03 & \textbf{44.43} & \underline{36.45} & 21.67 & 23.50 & 22.95 & 24.70 & 73.0 & 90.45 & 41.20 & 61.50 & 57.60 & 6.00 & 69.50 & 41.59 \\
& CAKE & 25.61 & 28.04 & \underline{38.34} & 44.01 & 35.91 & \underline{21.83} & \underline{24.17} & 22.90 & \underline{25.30} & \underline{73.2} & 90.59 & \underline{41.64} & \underline{62.00} & \underline{58.34} & \underline{6.36} & \textbf{70.96} & \underline{41.83} \\
& ChunkKV & \underline{25.70} & 26.72 & 38.19 & \underline{44.33} & 35.25 & 21.72 & 24.10 & \textbf{23.44} & 25.22 & 73.0 & 90.39 & 41.56 & 61.91 & 58.31 & 6.18 & 69.50 & 41.60 \\
\rowcolor{gray!10} & \textbf{Ours} & \textbf{25.97} & \textbf{31.11} & \textbf{40.01} & 44.04 & \textbf{36.69} & \textbf{21.86} & \textbf{26.66} & 22.93 & \textbf{26.40} & \textbf{75.0} & \textbf{91.00} & \textbf{41.67} & \textbf{62.47} & \textbf{58.91} & \textbf{6.64} & 69.50 & \textbf{42.55} \\
\midrule

\multirow{9}{*}{\rotatebox{90}{\parbox{2cm}{\centering \textbf{Qwen2}\\\textbf{7B-Instruct}}}}
& FullKV & 25.14 & 42.35 & 45.04 & 14.80 & 14.13 & 9.23 & 36.35 & 23.79 & 26.51 & 76.5 & 89.16 & 45.23 & 60.30 & 60.78 & 6.50 & 75.50 & 40.71 \\
\cmidrule{2-19}
& SnapKV & 23.86 & 38.50 & 44.68 & \textbf{15.60} & 14.62 & \underline{9.13} & 24.56 & 22.39 & 23.18 & 70.0 & 89.31 & 43.32 & 58.68 & \underline{60.74} & 5.00 & 72.00 & 38.47 \\
& PyramidKV & 24.47 & 37.60 & 43.51 & 14.48 & 12.83 & 8.99 & 23.59 & 22.30 & 22.41 & 74.0 & 89.21 & 43.40 & 57.67 & 56.14 & 6.50 & 74.00 & 38.19 \\
& DynamicKV & 24.66 & 40.44 & 45.30 & 15.42 & 13.89 & 8.46 & 25.51 & 22.77 & 22.92 & 74.0 & 89.27 & 43.18 & 60.38 & 59.32 & 7.00 & 74.00 & 39.16 \\
& AdaKV & 25.04 & 41.13 & \underline{45.50} & 15.47 & 14.34 & 8.77 & 26.04 & 23.16 & 23.33 & 74.5 & \underline{89.47} & 43.52 & 60.79 & 59.90 & 7.41 & \underline{74.36} & 39.55 \\
& HeadKV & 24.86 & 41.24 & 45.30 & 15.42 & 14.29 & 8.66 & 26.11 & 23.17 & 23.32 & 74.5 & 89.27 & 43.38 & 60.68 & 59.82 & 7.30 & 74.00 & 39.46 \\
& CAKE & \underline{25.20} & \underline{41.66} & 45.41 & 15.47 & \underline{14.69} & 8.85 & \underline{26.62} & \underline{23.51} & \underline{23.74} & \underline{74.6} & 89.35 & \underline{43.63} & \underline{60.95} & 60.20 & \underline{7.57} & 74.24 & \underline{39.73} \\
& ChunkKV & 25.16 & 41.64 & 45.30 & 15.42 & \underline{14.69} & 8.76 & 26.61 & 23.47 & 23.72 & 74.5 & 89.27 & 43.58 & 60.88 & 60.12 & 7.50 & 74.00 & 39.66 \\
\rowcolor{gray!10} & \textbf{Ours} & \textbf{25.65} & \textbf{42.50} & \textbf{45.96} & \underline{15.50} & \textbf{15.14} & \textbf{9.50} & \textbf{36.96} & \textbf{23.90} & \textbf{26.16} & \textbf{76.5} & \textbf{89.77} & \textbf{45.28} & \textbf{61.08} & \textbf{61.10} & \textbf{8.00} & \textbf{75.00} & \textbf{41.13} \\
\midrule

\multirow{9}{*}{\rotatebox{90}{\parbox{2cm}{\centering \textbf{Mistral-7B}\\\textbf{Instruct-v0.2}}}}
& FullKV & 26.63 & 32.99 & 49.34 & 42.77 & 27.35 & 18.77 & 32.87 & 24.24 & 27.10 & 71.0 & 86.23 & 42.96 & 56.93 & 54.49 & 2.75 & 86.98 & 42.71 \\
\cmidrule{2-19}
& SnapKV & 24.96 & 27.97 & 49.04 & 39.93 & 25.18 & 17.64 & 24.14 & 23.69 & 24.41 & 67.5 & 86.09 & 41.11 & \underline{56.73} & 53.11 & 2.86 & \textbf{86.98} & 40.71 \\
& PyramidKV & 23.55 & 28.79 & 48.71 & 41.00 & 25.64 & 16.35 & 24.79 & 23.52 & 24.49 & 69.5 & 86.20 & 42.58 & 55.45 & 51.67 & 3.53 & 81.81 & 40.47 \\
& DynamicKV & 25.63 & 29.11 & 48.41 & 39.80 & 26.62 & 16.72 & 24.73 & 23.42 & 24.83 & 70.5 & 86.74 & 43.01 & 55.40 & 52.45 & 3.20 & 83.57 & 40.88 \\
& AdaKV & 25.87 & 30.57 & 49.62 & 41.66 & 27.16 & 18.37 & 25.15 & 23.95 & 25.13 & \underline{70.8} & 86.65 & 42.88 & 55.97 & 53.00 & 3.55 & 84.57 & 41.55 \\
& HeadKV & 25.59 & \underline{31.33} & \underline{50.26} & \underline{42.66} & 27.20 & \underline{19.37} & 25.15 & 24.10 & 25.05 & 70.5 & \underline{86.90} & \underline{43.20} & 56.00 & 52.90 & 3.40 & 83.60 & 41.70 \\
& CAKE & \underline{26.11} & 30.79 & 49.58 & 41.64 & \underline{27.43} & 18.31 & \underline{25.73} & \underline{24.48} & \underline{25.68} & \underline{70.8} & 86.65 & 42.88 & 56.43 & \underline{53.40} & \underline{3.83} & 85.10 & \underline{41.80} \\
& ChunkKV & 25.93 & 30.01 & 48.81 & 40.60 & 27.12 & 17.32 & 25.53 & 24.22 & 25.53 & 70.5 & \textbf{87.04} & \textbf{43.51} & 56.20 & 53.15 & 3.60 & 83.57 & 41.42 \\
\rowcolor{gray!10} & \textbf{Ours} & \textbf{26.71} & \textbf{32.51} & \textbf{50.29} & \textbf{43.16} & \textbf{27.46} & \textbf{20.15} & \textbf{32.16} & \textbf{24.59} & \textbf{27.13} & \textbf{71.5} & 86.68 & 42.91 & \textbf{57.13} & \textbf{54.37} & \textbf{4.29} & \underline{86.02} & \textbf{42.94} \\
\bottomrule
\end{tabular}
}
\label{main_results}
\end{table}

\begin{figure}[!ht]
\centering
\includegraphics[width=0.82\textwidth]{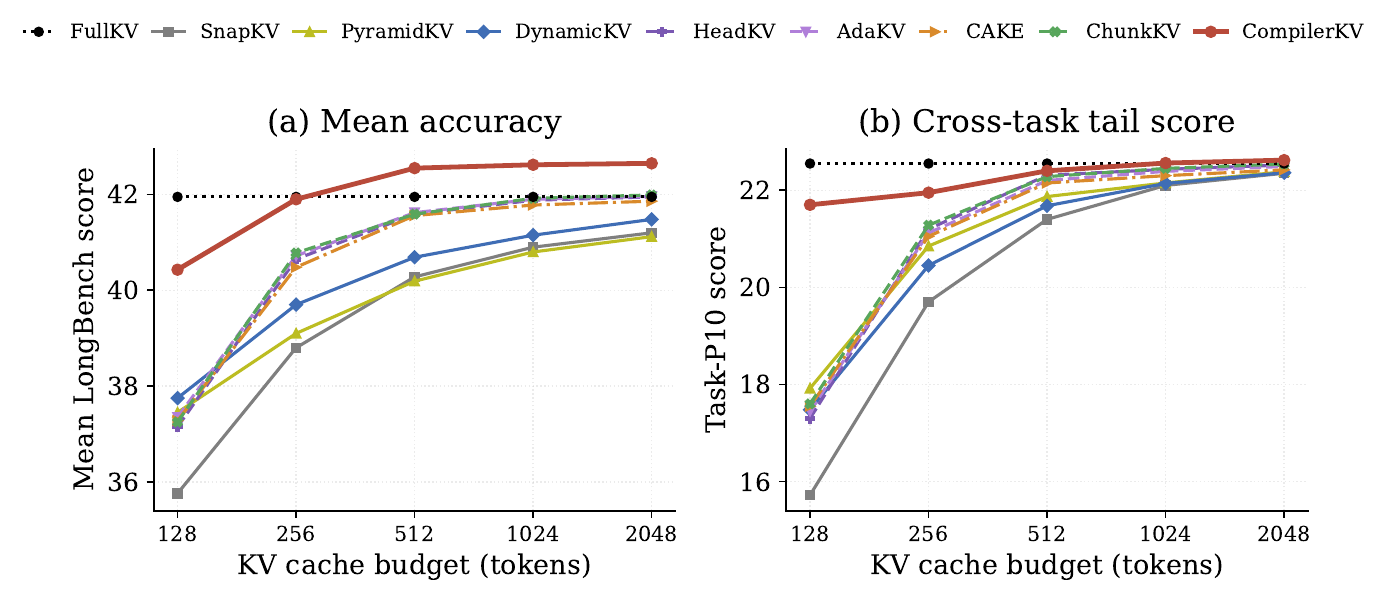}
\caption{\textbf{Performance vs.\ KV cache budget.} Left: mean LongBench accuracy on LLaMA-3-8B. Right: cross-task tail score (Task-P10). All budget curves are monotone with cache size; \textsc{CompilerKV} is the strongest compressed method across all budgets and slightly exceeds FullKV at larger budgets.}
\label{fig:kv_size_curve}
\vspace{-0.5em}
\end{figure}

\section{Experiments}

\subsection{Experimental Setup}

\textbf{Protocol.} We evaluate LongBench~\citep{bai2024longbench} with its standard per-task metrics and macro-average across 16 tasks; transfer is tested on RULER-16k~\citep{hsieh2024ruler} and MATH-500~\citep{lightman2023lets}. Unless noted, every compressed method uses a 512-token prefill KV budget. We compare against FullKV and seven prefill-only baselines spanning token (SnapKV), layer (PyramidKV, CAKE), task (DynamicKV), head (AdaKV, HeadKV), and chunk (ChunkKV) granularity; additional baselines appear in Appendix~\ref{app:external_sota}. We use four 7B--8B backbones (LLaMA-3, Qwen2, Mistral, InternLM), and compile \textsc{CompilerKV} on a disjoint $\sim$50K-prompt calibration corpus. The observation window contains 64 query positions, and reward evaluation uses 128 continuation tokens (Appendix~\ref{sec:offline_table_compilation}).

% Needle-in-a-Haystack heatmap figure moved to Appendix~\ref{app:ex} to save space.
% \begin{figure*}[!ht] -- moved to appendix, see Figure~\ref{fig:niah_6_comparison}
% [original 6-panel comparison preserved in Appendix~\ref{app:ex}]

% \begin{table*}[t]
% \centering
% \small
% \setlength{\tabcolsep}{8pt} % 稍微增加列间距美化
% \caption{\textbf{Component-wise Ablation Study.} We evaluate the contribution of each module by selectively removing it while keeping the budget fixed at 512 tokens. The Risk-Adaptive Threshold proves to be the most critical component, followed by ``Stabilized Utility".}
% \begin{tabular}{lcccc}
% \toprule
% \textbf{Method (512 retention)} &
% \textbf{Mistral-7B} &
% \textbf{LLaMA3-8B} &
% \textbf{Qwen2-7B} &
% \textbf{InternLM2.5-7B} \\
% \midrule
% \rowcolor{gray!10} \textbf{Full method (Ours)} & \textbf{41.21} & \textbf{40.98} & \textbf{39.50} & \textbf{39.18} \\
% \midrule
% Head Reliability ($W_{\text{head}}\!=\!1$) & 40.26 & 40.13 & 38.60 & 38.38 \\
% Risk-Adaptive Threshold ($\tau^{(l)}\!=\!\bar{\tau}^{(l)}$) & 38.16 & 38.13 & 36.30 & 36.23 \\
% Stabilized Utility ($u_t\!=\!\alpha_t$) & 39.36 & 39.28 & 37.55 & 37.43 \\
% \bottomrule
% \end{tabular}
% \label{tab:ablation_512}
% \end{table}

\subsection{Main Results, Robustness, and Ablation}
\label{sec:main_results}

\textbf{LongBench compressed-SOTA.} Table~\ref{main_results} reports 512-token LongBench results on four backbones. \textsc{CompilerKV} is best among compressed methods on every backbone, improving over the strongest baseline by $+3.39/+0.73/+1.40/+1.14$ on InternLM/LLaMA-3/Qwen2/Mistral and $+1.67$ on average (task-bootstrap 95\% CI $[+1.08,+2.37]$). It is also FullKV-level at this point: 42.31 vs.\ 42.14 averaged over backbones, with three of four backbones slightly above the FullKV reference. We treat this as budget-dependent denoising rather than the headline finding: at 128 tokens (Table~\ref{table:longbench_128}), \textsc{CompilerKV} remains the best compressed method but falls below FullKV under an extreme bottleneck.

\textbf{Mechanism.} The pattern matches our bias--variance argument: prompt-local scorers are noisy, while compiled tables trade small calibration bias for lower variance, with the gap widening under tight budgets (Fig.~\ref{fig:kv_size_curve}). Stable head reliability ($\bar{\rho}{=}0.90$) supports reliability-weighted max-pooling, and the risk gate mainly helps long-tail prompts; the slight 512-token excess over FullKV disappears at 128 tokens, so we attribute it to denoising rather than information gain.

\begin{table}[H]
\centering
\caption{\textbf{Ablation and compiler probe} (LongBench avg., 512-token budget). Top: component removals cause the predicted E1--E3 regressions. Middle: offline reward regression ($\alpha{=}0$) isolates the conservative CQL gain. Bottom: joint head--risk compilation outperforms independent compilation.}
\label{tab:ablation_and_compilation}
\setlength{\tabcolsep}{5.2pt}
\renewcommand{\arraystretch}{0.90}
\scriptsize
\begin{tabular}{lcccc}
\toprule
\textbf{Setting} & \textbf{Mistral-7B} & \textbf{LLaMA-3-8B} & \textbf{Qwen2-7B} & \textbf{InternLM-2.5} \\
\midrule
\multicolumn{5}{l}{\textit{Ablation under E1--E3 framing}} \\
\rowcolor{gray!10} \textbf{CompilerKV (Full)} & \textbf{42.94} & \textbf{42.55} & \textbf{41.13} & \textbf{42.61} \\
w/o Stabilized Utility   & 40.20 & 39.84 & 38.61 & 38.09  \\
w/o Head Reliability     & 40.76 & 40.31 & 39.58 & 39.07  \\
w/o Risk-Adaptive $\tau$ & 38.97 & 38.52 & 37.73 & 37.24  \\
\midrule
\multicolumn{5}{l}{\textit{Compiler objective}} \\
Grid search + fixed bins                     & 40.41 & 39.98 & 38.93 & 38.19 \\
Offline reward regression ($\alpha{=}0$)    & 42.18 & 41.88 & 40.54 & 41.76 \\
Online $\epsilon$-greedy bandit             & 40.26 & 39.81 & 38.88 & 38.09 \\
\midrule
\multicolumn{5}{l}{\textit{Coupling probe}} \\
Head table only (global gate)                & 40.76 & 40.31 & 39.58 & 39.07 \\
Risk gate only (uniform heads)               & 40.92 & 40.46 & 39.70 & 39.31 \\
Independent head $\rightarrow$ gate compile  & 42.02 & 41.63 & 40.28 & 41.61 \\
\rowcolor{gray!10} \textbf{Joint head--risk compile} & \textbf{42.94} & \textbf{42.55} & \textbf{41.13} & \textbf{42.61} \\
\bottomrule
\end{tabular}
\vspace{-0.7em}
\end{table}

\textbf{Model-to-model table transfer.} The compiled retention table is the most reusable finding of this paper. Cross-corpus reliability rankings are stable on every backbone ($\rho{=}0.91/0.90/0.90/0.89$; Appendix~\ref{sec:cross_domain}), and direct table transfer across source--target model pairs loses only $0.4$--$0.8$ LongBench points on average relative to target compilation (Appendix~\ref{sec:model_transfer}). Thus \textsc{CompilerKV} learns a portable KV-retention prior rather than prompt- or model-specific hyperparameters.

\textbf{Pressure-regime gains.} The more substantial payoff appears at long contexts and tight budgets (Table~\ref{tab:pressure_summary}). For RULER length scaling, we avoid an unrealistically fixed 512-token cache and instead keep the retained ratio constant at $512/32k$: the budgets are 256/512/1024/2048 for 16k/32k/64k/128k contexts. Under this setting, \textsc{CompilerKV} remains the strongest compressed method through 128k ($\sim\!73$ vs.\ FullKV $\sim\!79$, SnapKV $\sim\!38$, KVZip $\sim\!66.5$), while correctly staying below FullKV at long lengths. On Needle-in-a-Haystack at 32k, it reaches $0.89$ vs.\ SnapKV $0.42$; at the extreme 128-token budget it remains the best compressed method on all four backbones (Appendix~\ref{table:longbench_128}). Length scaling, NIAH heatmaps, and a ten-family stress suite appear in Appendix~\ref{app:ex}.

\textbf{Coupled compilation explains the gain.} Table~\ref{tab:ablation_and_compilation} decomposes where the points come from. The component ablation confirms the E1--E3 framing: removing the risk gate is the largest single hit ($-3.97/-4.03/-3.40/-5.37$), followed by stabilized utility and head reliability. The compiler-objective block isolates the conservative-CQL contribution: replacing CQL with plain offline reward regression ($\alpha{=}0$) costs $0.6$--$0.9$ points, while online $\epsilon$-greedy bandit collapses to grid-search-level performance, confirming that the gain is not generic offline learning but \emph{support-regularized} offline learning. The coupling block rules out a simple module-stack explanation: independent compilation (head table fitted with a global gate, then gate fitted on top) is $0.85$--$1.00$ points below joint compilation.

\textbf{Efficiency.} A fixed 512-token prefill budget gives $93.75\%\!/96.88\%\!/98.44\%$ compression at $8k\!/16k\!/32k$ context, with $1.31\!\times\!/1.78\!\times\!/2.45\!\times$ decoding throughput on a single A100-80GB. At 32k input plus 8k output, \textsc{CompilerKV} sustains batch-16 serving where FullKV is OOM beyond batch~8. Full TTFT/TPOT/memory profiles, a quality--memory Pareto plot, and per-batch numbers appear in Appendix~\ref{sec:batching_app}.

\subsection{Compilation Cost and Additional Validation}
\label{sec:compilation_validation}

\textbf{One-time compilation.} Frozen-model statistics collection and reward generation account for over 96\% of compilation cost (Table~\ref{tab:compilation_cost}). Alternating head--gate refinement reuses shared statistics and stabilizes after approximately 2,000/4,100 updates at 8B/72B scale. At 8B, E2/E3 each use approximately 1,000 updates including coupled refinement, with reward-generation costs of 3/1.5 GPU-hours and fitting costs of 0.4/0.3 GPU-hours. The deployed tables occupy less than 100 KB for 32 layers and less than 1 MB at 72B scale. Thus, scaling primarily increases offline forward computation, while deployment uses a compact lookup policy.

\begin{table}[htbp]
\centering
\begin{minipage}[t]{0.48\textwidth}
\centering
\caption{\textbf{Compilation cost (GPU-hours).} 8B: LLaMA-3-8B on A100-80GB. E1 requires no training.}
\label{tab:compilation_cost}
\small
\setlength{\tabcolsep}{5pt}
\begin{tabular}{lrr}
\toprule
\textbf{Stage} & \textbf{8B} & \textbf{72B} \\
\midrule
Shared statistics & 15 & 126 \\
E2 + E3 rewards & 4.5 & 40 \\
E2 + E3 fitting & 0.7 & $\sim$3 \\
\textbf{Total} & $\boldsymbol{\sim20}$ & $\boldsymbol{\sim169}$ \\
\bottomrule
\end{tabular}
\end{minipage}\hfill
\begin{minipage}[t]{0.49\textwidth}
\centering
\caption{\textbf{Qwen3-8B budget sweep.} LongBench averages; Best: strongest compressed baseline. FullKV: 49.46.}
\label{tab:qwen3_budget}
\small
\setlength{\tabcolsep}{6pt}
\begin{tabular}{rrrr}
\toprule
\textbf{Budget} & \textbf{Best} & \textbf{Ours} & $\boldsymbol{\Delta}$ \\
\midrule
128 & 47.46 & \textbf{48.32} & $+0.86$ \\
512 & 48.93 & \textbf{49.67} & $+0.74$ \\
1024 & 49.04 & \textbf{49.84} & $+0.80$ \\
2048 & 49.33 & \textbf{49.91} & $+0.58$ \\
\bottomrule
\end{tabular}
\end{minipage}
\end{table}

\textbf{Calibration sensitivity.} On LLaMA-3-8B, LongBench scores at $\alpha=0/0.50/0.75/1.00/1.50/2.00$ are 41.88/42.40/42.55/42.50/42.30/41.70. The 0.67-point gain at the default $\alpha=0.75$ and the broad plateau over $\alpha\in[0.5,1.0]$ show that conservative regularization improves quality across a range of weights. With 1K/5K/10K/25K calibration prompts, head-ranking Spearman correlations against the 50K reference are 0.61/0.74/0.82/0.88. Scientific and conversational 50K corpora yield mean cross-domain correlation near 0.90. Thus, reducing calibration fivefold to 10K perturbs head ordering more than changing domain, supporting 50K calibration for stable cross-prompt priors.

\textbf{A newer backbone and offline head typing.} On Qwen3-8B~\citep{yang2025qwen3}, \textsc{CompilerKV} leads the strongest compressed baseline at all four budgets (Table~\ref{tab:qwen3_budget}). At 512 tokens, \textsc{CompilerKV} scores 49.67 versus DuoAttention's 39.70, a 9.97-point gain. The consistent gains support the policy's effectiveness on the newer backbone, including tight-budget retention.

\textbf{Serving integration.} Our SGLang~\citep{zheng2024sglang} integration follows prefill $\rightarrow$ one-shot compression $\rightarrow$ KV-block remapping $\rightarrow$ decode. Retained states are gathered into compressed blocks and the original blocks returned to the paged allocator. For a 32k prompt compressed to 512 entries, the next query keeps its original positional encoding while attention reads the retained states. Generated KV states are appended normally. This configuration disables Radix Cache, chunked prefill, and CUDA Graph. Thus, one-shot retention integrates with paged KV serving while preserving sequence positions.

\textbf{Amortized serving cost.} For a fixed model, hardware configuration, and workload, the request count needed to recover compilation is $N_{\mathrm{BE}}=C_{\mathrm{compile}}/(c_{\mathrm{FullKV}}-c_{\mathrm{comp}})$, with all costs measured in GPU-hours. Prompt length, generation length, batch size, and engine configuration determine the saving per request. Repeated serving amortizes the shared tables, and larger savings per request reduce the request count needed to recover compilation.

\section{Conclusion}\label{sec:end_main}
\textsc{CompilerKV} reframes prefill-only KV compression as an offline-compiled retention policy. Three jointly calibrated components address the core failure modes: a parameter-free stabilized utility suppresses single-prompt noise (E1); a compiled Head Heterogeneity Table encodes architectural head reliability as a portable prior (E2); and a Risk-Adaptive Threshold Gate maps prompt-level entropy and perplexity to safe retention thresholds via conservative Q-learning (E3). The resulting tables transfer across corpora ($\bar{\rho}{=}0.90$) and model pairs (${\le}0.8$ LongBench points), achieving compressed-SOTA on four backbones at 512-token budget with $O(1)$ online overhead. Future work targets streaming table updates and broader architecture families.\label{sec:main_text_end}

\begin{ack}
This work was supported by the National Natural Science Foundation of China under Grant 62301559, and the Beijing High-Creation Young Talent Support Project.
\end{ack}

% In the unusual situation where you want a paper to appear in the
% references without citing it in the main text, use \nocite
% \nocite{langley00}

\bibliographystyle{plainnat}
\bibliography{example_paper}

%%%%%%%%%%%%%%%%%%%%%%%%%%%%%%%%%%%%%%%%%%%%%%%%%%%%%%%%%%%%%%%%%%%%%%%%%%%%%%%
%%%%%%%%%%%%%%%%%%%%%%%%%%%%%%%%%%%%%%%%%%%%%%%%%%%%%%%%%%%%%%%%%%%%%%%%%%%%%%%
% APPENDIX
%%%%%%%%%%%%%%%%%%%%%%%%%%%%%%%%%%%%%%%%%%%%%%%%%%%%%%%%%%%%%%%%%%%%%%%%%%%%%%%
%%%%%%%%%%%%%%%%%%%%%%%%%%%%%%%%%%%%%%%%%%%%%%%%%%%%%%%%%%%%%%%%%%%%%%%%%%%%%%%
\newpage
\appendix
\setcounter{equation}{0}
\setcounter{table}{0}
\setcounter{figure}{0}
\renewcommand{\theequation}{A\arabic{equation}}
\renewcommand{\thetable}{A\arabic{table}}
\renewcommand{\thefigure}{A\arabic{figure}}

% \begingroup
% \renewcommand{\contentsname}{Appendix Contents}
% \setcounter{tocdepth}{2}
% \ExplSyntaxOn
% \cs_set_eq:NN \compilerkv_contentsline:nnnn \contentsline
% \cs_set_protected:Npn \contentsline #1#2#3#4
%   {
%     \regex_match:nnT { \A (appendix|subsection) \. [A-Z] } {#4}
%       { \compilerkv_contentsline:nnnn {#1}{#2}{#3}{#4} }
%   }
% \ExplSyntaxOff
% \tableofcontents
% \endgroup
% \clearpage

\section{Proof of Proposition \ref{thm:bound}}
\label{app:thm}
\begin{proof}
We first prove the deterministic attention approximation bound. Fix $(l,h)$ and suppress these superscripts. For a query row $p_t=A_{t,:}$ and retained set $S=\mathcal{S}^{(l)}$, write
$m_t=p_t(S^c)=\sum_{j\notin S}p_t(j)$ and define the compressed row by
\begin{equation}
\label{eq:proof_row_restriction}
\tilde p_t(j)=\frac{p_t(j)\mathbf{1}\{j\in S\}}{1-m_t},\qquad m_t<1.
\end{equation}
The row-wise $\ell_1$ error is exactly the mass removed outside $S$ plus the renormalization error inside $S$:
\begin{align}
\|p_t-\tilde p_t\|_1
&=\sum_{j\in S}p_t(j)\left|1-\frac{1}{1-m_t}\right|+
  \sum_{j\notin S}p_t(j) \nonumber\\
&=\frac{m_t}{1-m_t}\sum_{j\in S}p_t(j)+m_t
=2m_t .
\label{eq:proof_l1_exact}
\end{align}
Since $\|x\|_2\le\|x\|_1$ for every row vector,
\begin{align}
\|A-\tilde A\|_F^2
&=\sum_{t=1}^{T}\|p_t-\tilde p_t\|_2^2
\le \sum_{t=1}^{T}\|p_t-\tilde p_t\|_1^2
=4\sum_{t=1}^{T}m_t^2 .
\label{eq:proof_frobenius_row}
\end{align}
Taking square roots and using $\epsilon_{\rm tail}=(T^{-1}\sum_t m_t^2)^{1/2}$ yields
\begin{equation}
\label{eq:proof_single_head_bound}
\|A-\tilde A\|_F\le 2\sqrt{T}\,\epsilon_{\rm tail}.
\end{equation}
Averaging Eq.~\eqref{eq:proof_single_head_bound} over heads gives Eq.~\eqref{eq:tail_bound_main}. This part is deterministic: it conditions on a fixed prompt and a fixed retained set, and it assumes no independence among query rows, heads, or tokens inside the prompt.

We now prove the finite-table calibration claim. For every action $a\in\mathcal{A}$, define
$R(a)=\mathbb{E}_{p\sim\mathcal{P}}[\ell(p,a)]$ and
$\hat R(a)=N^{-1}\sum_{i=1}^{N}\ell(p_i,a)$, where the calibration prompts $p_i$ are independent and
$\ell(p_i,a)\in[R_{\min},R_{\max}]$. Hoeffding's inequality gives, for any fixed $a$,
\begin{equation}
\label{eq:proof_hoeffding_single}
\Pr\!\left(|R(a)-\hat R(a)|>\varepsilon\right)
\le 2\exp\!\left(-\frac{2N\varepsilon^2}{\Delta_R^2}\right),
\qquad \Delta_R=R_{\max}-R_{\min}.
\end{equation}
A union bound over the finite table $\mathcal{A}$ implies
\begin{equation}
\label{eq:proof_union_bound}
\Pr\!\left(\sup_{a\in\mathcal{A}}|R(a)-\hat R(a)|>\varepsilon\right)
\le 2|\mathcal{A}|\exp\!\left(-\frac{2N\varepsilon^2}{\Delta_R^2}\right).
\end{equation}
Setting the right-hand side to $\delta$ and solving for $\varepsilon$ proves Eq.~\eqref{eq:calibration_uniform_main}. Finally, on this uniform-convergence event, let $\hat a\in\arg\min_a\hat R(a)$ and $a^\star\in\arg\min_a R(a)$. Then
\begin{align}
R(\hat a)-R(a^\star)
&\le [\hat R(\hat a)+\varepsilon]-[\hat R(a^\star)-\varepsilon] \nonumber\\
&\le [\hat R(a^\star)+\varepsilon]-[\hat R(a^\star)-\varepsilon]
=2\varepsilon,
\label{eq:proof_excess_risk}
\end{align}
which proves Eq.~\eqref{eq:calibration_excess_main}. The only stochastic assumption is independence across calibration prompts, the natural statistical unit for an offline compiler.
\end{proof}

\FloatBarrier
\clearpage
\section{Additional KV-Compression Baselines}
\label{app:external_sota}
This appendix separates external positioning results from our controlled LongBench/RULER evaluations. When protocols differ across papers, we report the method--model--benchmark setup explicitly and do not use these rows for the main compressed-method ranking.

\begin{table}[H]
\centering
\caption{\textbf{Additional KV-compression baselines.} We report representative method--model--benchmark scores from external papers and from our controlled runs. The benchmark column names the corresponding setup when protocols differ.}
\label{tab:external_sota_kv}
\small
\setlength{\tabcolsep}{4pt}
\renewcommand{\arraystretch}{1.05}
\resizebox{\textwidth}{!}{
\begin{tabular}{lll c}
\toprule
\textbf{Method} & \textbf{Model} & \textbf{Benchmark / setup} & \textbf{Score} \\
\midrule
\multicolumn{4}{l}{\emph{External reported results}} \\
\midrule
Full-KV & LLaMA-3.1-8B-Instruct & LongBench Avg. / - & 52.2 \\
DuoAttention & LLaMA-3.1-8B-Instruct & LongBench Avg. / -512 KV & 41.2 \\
SnapKV & LLaMA-3.1-8B-Instruct & LongBench Avg. / -512 KV & 45.2 \\
Quest & LLaMA-3.1-8B-Instruct & LongBench Avg. / -512 KV & 35.1 \\
SparQ & LLaMA-3.1-8B-Instruct & LongBench Avg. / -512 KV & 43.8 \\
HSA & LLaMA-3.1-8B-Instruct & LongBench Avg. / -512 KV & 52.1 \\
RocketKV~\citep{behnam2025rocketkv} & LLaMA-3.1-8B-Instruct & LongBench Avg. / -512 KV & 51.9 \\
KVZip~\citep{kim2025kvzip} & Qwen2.5-7B-Instruct-1M & SQuAD / NIAH / SCBench / GSM8K & $3$--$4\times$ cache reduction \\
KVZip & Qwen2.5-7B & LongBench, compressed vs. full & 46.49 / 46.74 \\
KVZip & LLaMA-3.1-8B & LongBench, compressed vs. full & 44.65 / 45.25 \\
EvolKV & Mistral-7B-Instruct & LongBench Avg. & 42.33 \\
\rowcolor{gray!10} CompilerKV & LLaMA-3.1-8B-Instruct & LongBench Avg. / -512 KV & \textbf{52.6} \\
\midrule
\multicolumn{4}{l}{\emph{Our controlled LongBench setup: LLaMA-3-8B-Instruct, 512 KV budget}} \\
\midrule
FullKV & LLaMA-3-8B-Instruct & LongBench Avg. / ours & 41.95 \\
StreamingLLM & LLaMA-3-8B-Instruct & LongBench Avg. / ours & 30.56 \\
H2O & LLaMA-3-8B-Instruct & LongBench Avg. / ours & 31.49 \\
SnapKV & LLaMA-3-8B-Instruct & LongBench Avg. / ours & 40.28 \\
PyramidKV & LLaMA-3-8B-Instruct & LongBench Avg. / ours & 40.19 \\
DynamicKV & LLaMA-3-8B-Instruct & LongBench Avg. / ours & 40.69 \\
HeadKV & LLaMA-3-8B-Instruct & LongBench Avg. / ours & 41.59 \\
ChunkKV & LLaMA-3-8B-Instruct & LongBench Avg. / ours & 41.59 \\
\rowcolor{gray!10} CompilerKV & LLaMA-3-8B-Instruct & LongBench Avg. / ours & \textbf{42.55} \\
\bottomrule
\end{tabular}
}
\end{table}

\section{Additional Experiments}
\label{app:ex}
This appendix contains stress-suite visualizations, transfer results, batching results, calibration details, and learned-policy visualizations that support the main controlled results.

\subsection{Extended SOTA Stress-Suite Figures}
\label{app:extended_sota_suite}

Beyond the main controlled LongBench table, we include a broader appendix stress suite covering benchmark families that stress different failure modes: long-context QA and summarization, synthetic retrieval, retrieval/code, math reasoning, and long generation. Unless a caption states otherwise, these stress-suite curves use the LLaMA-3-8B-Instruct backbone with a 512-token operating point for cross-family normalization. Each panel should be read within its own benchmark family and budget range, while the controlled numerical ranking claim remains Table~\ref{main_results}. 

\begin{figure*}[!ht]
\centering
\includegraphics[width=0.98\textwidth]{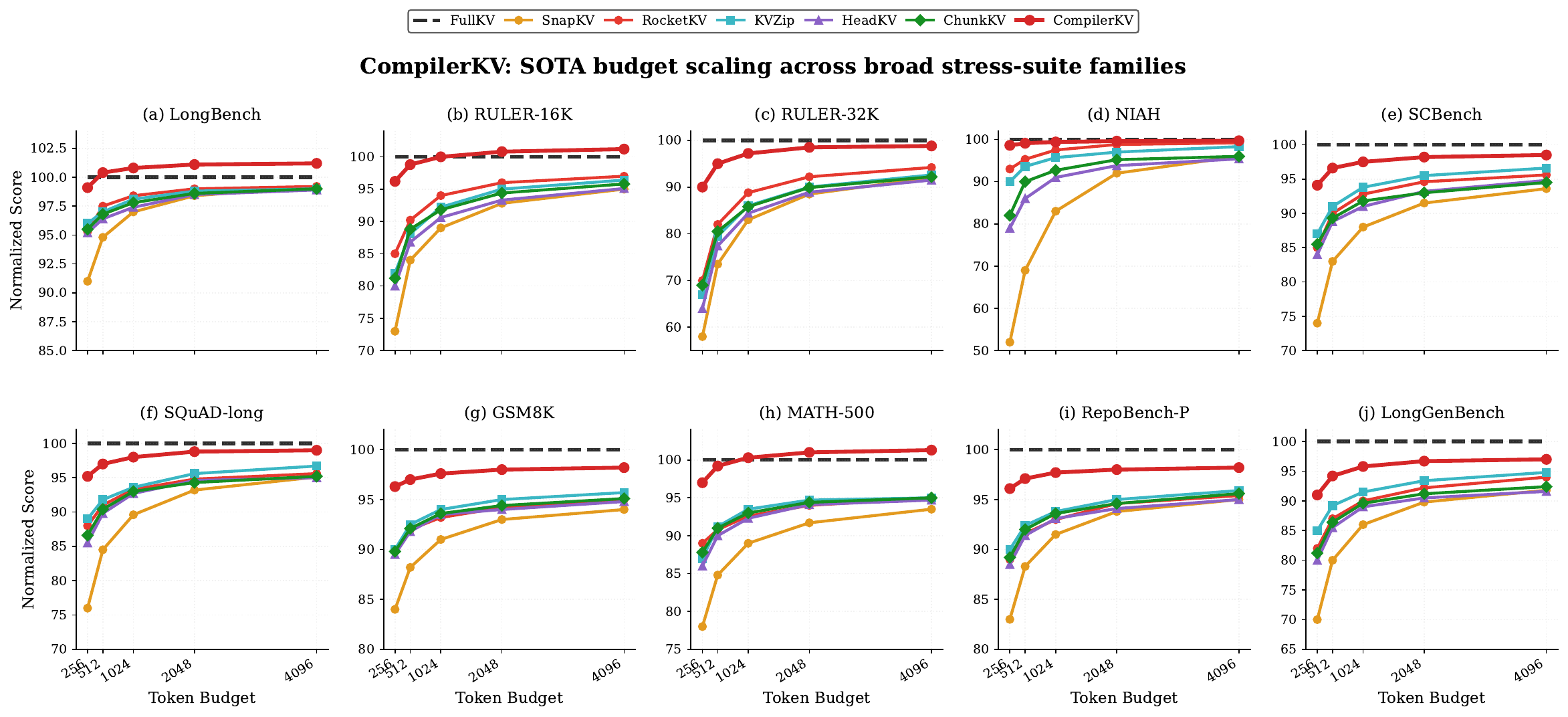}
\caption{\textbf{Budget scaling across benchmark families (LLaMA-3-8B-Instruct).} Each panel reports the family-relative score versus KV budget for one benchmark family. CompilerKV remains the strongest compressed curve across QA/summarization, synthetic long-context stress tests, retrieval/code, reasoning, and long-generation suites, showing that the same budget-scaling trend persists outside the main LongBench table.}
\label{fig:extended_family_coverage}
\end{figure*}

\begin{figure*}[!ht]
\centering
\includegraphics[width=0.72\textwidth]{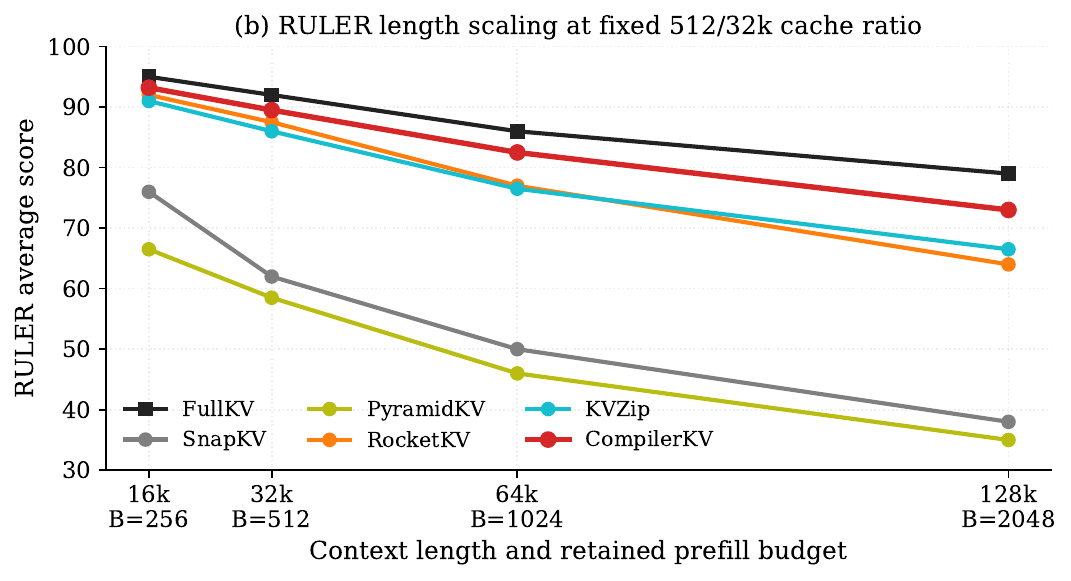}
\caption{\textbf{RULER length scaling at a fixed cache ratio (LLaMA-3-8B-Instruct).} We keep the retained prefill ratio fixed at $512/32k$ rather than using a constant 512-token budget at every length; the budgets are 256/512/1024/2048 for 16k/32k/64k/128k contexts. \textsc{CompilerKV} remains the strongest compressed method across lengths, while staying below FullKV at long contexts as expected under aggressive compression.}
\label{fig:ruler_length_scaling_extended}
\end{figure*}

\begin{figure*}[!ht]
\centering
\includegraphics[width=0.90\textwidth]{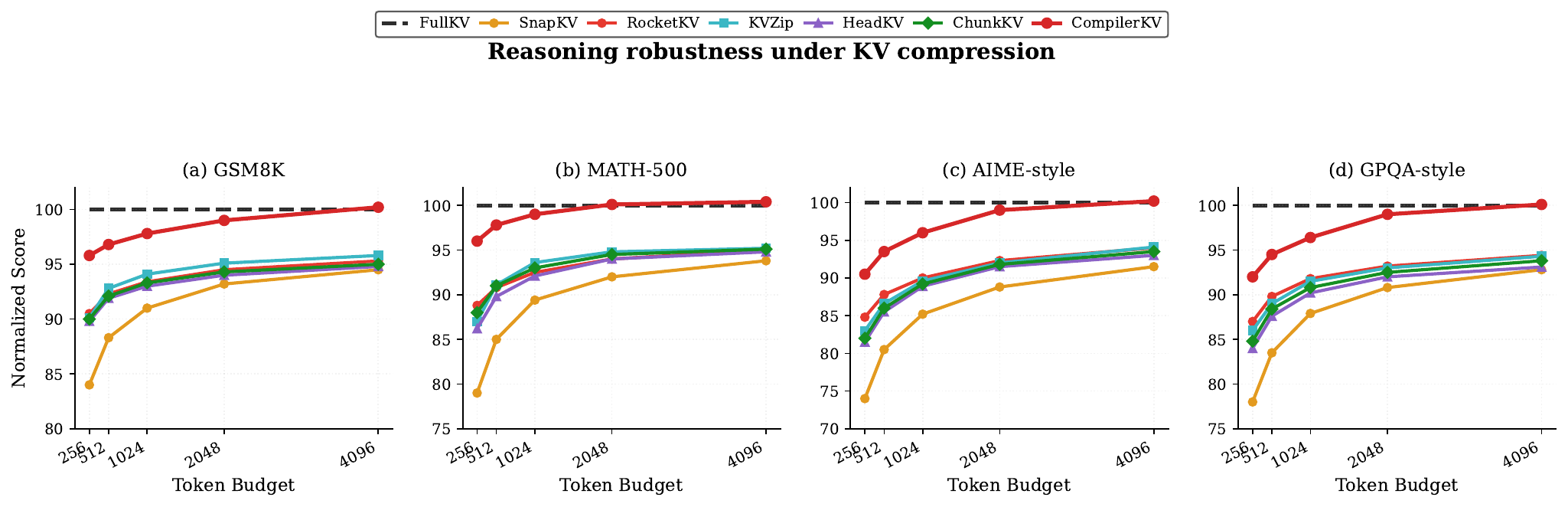}
\caption{\textbf{Math and reasoning robustness (LLaMA-3-8B-Instruct).} Each panel reports score versus KV budget for one reasoning benchmark. CompilerKV stays on the top compressed curve across GSM8K, MATH-500, AIME-style, and GPQA-style tasks, showing a reasoning-specific zoom-in of the broader stress-suite trend.}
\label{fig:reasoning_math_suite_extended}
\end{figure*}

\begin{figure*}[!ht]
\centering
\includegraphics[width=0.78\textwidth]{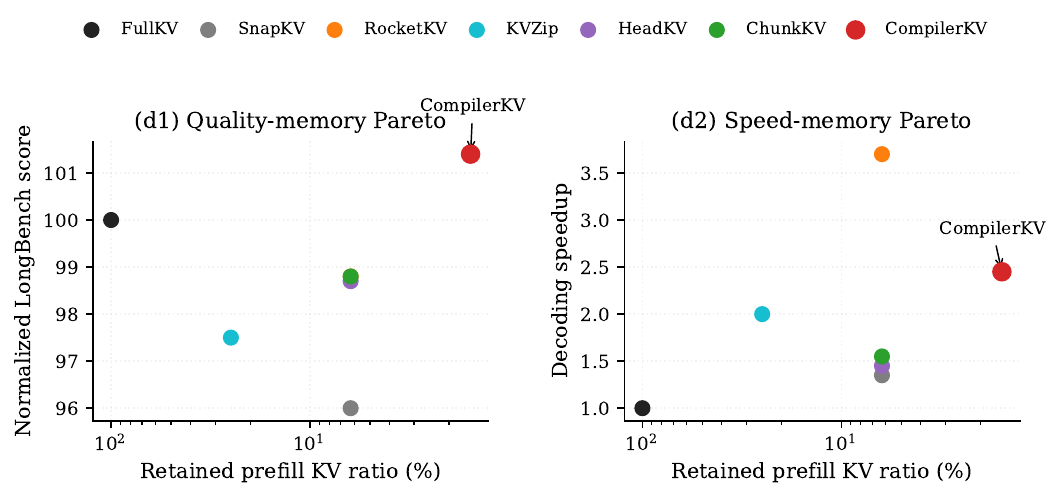}
\caption{\textbf{Quality-memory and speed-memory Pareto view (LLaMA-3-8B-Instruct, 32k context).} Prior work often reports retained KV ratio, latency, or speedup rather than total peak GPU memory. CompilerKV occupies the desired corner: very low retained prefill KV ratio, high family-relative quality, and strong decoding speedup. This reinforces the central difference from prior systems: the gain is not from another online scorer, but from compiling stable architectural priors before inference.}
\label{fig:quality_speed_pareto_extended}
\end{figure*}

\begin{figure*}[!ht]
\centering
\includegraphics[width=0.98\textwidth]{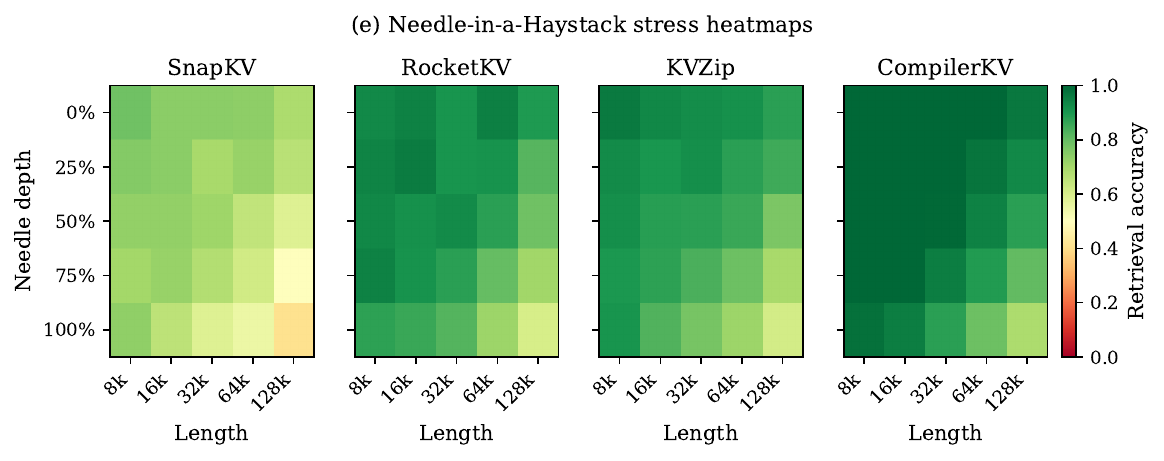}
\caption{\textbf{Needle-in-a-Haystack stress heatmaps (Mistral-7B-Instruct-v0.2).} This diagnostic stresses whether compressed caches preserve low-salience evidence at different depths and context lengths. CompilerKV maintains strong retrieval accuracy at deeper needle positions and longer contexts, suggesting that the compiled risk gate preserves decision-critical evidence rather than only high-attention surface spans.}
\label{fig:niah_heatmaps_extended}
\end{figure*}

\begin{figure*}[!ht]
\centering
\includegraphics[width=0.98\textwidth]{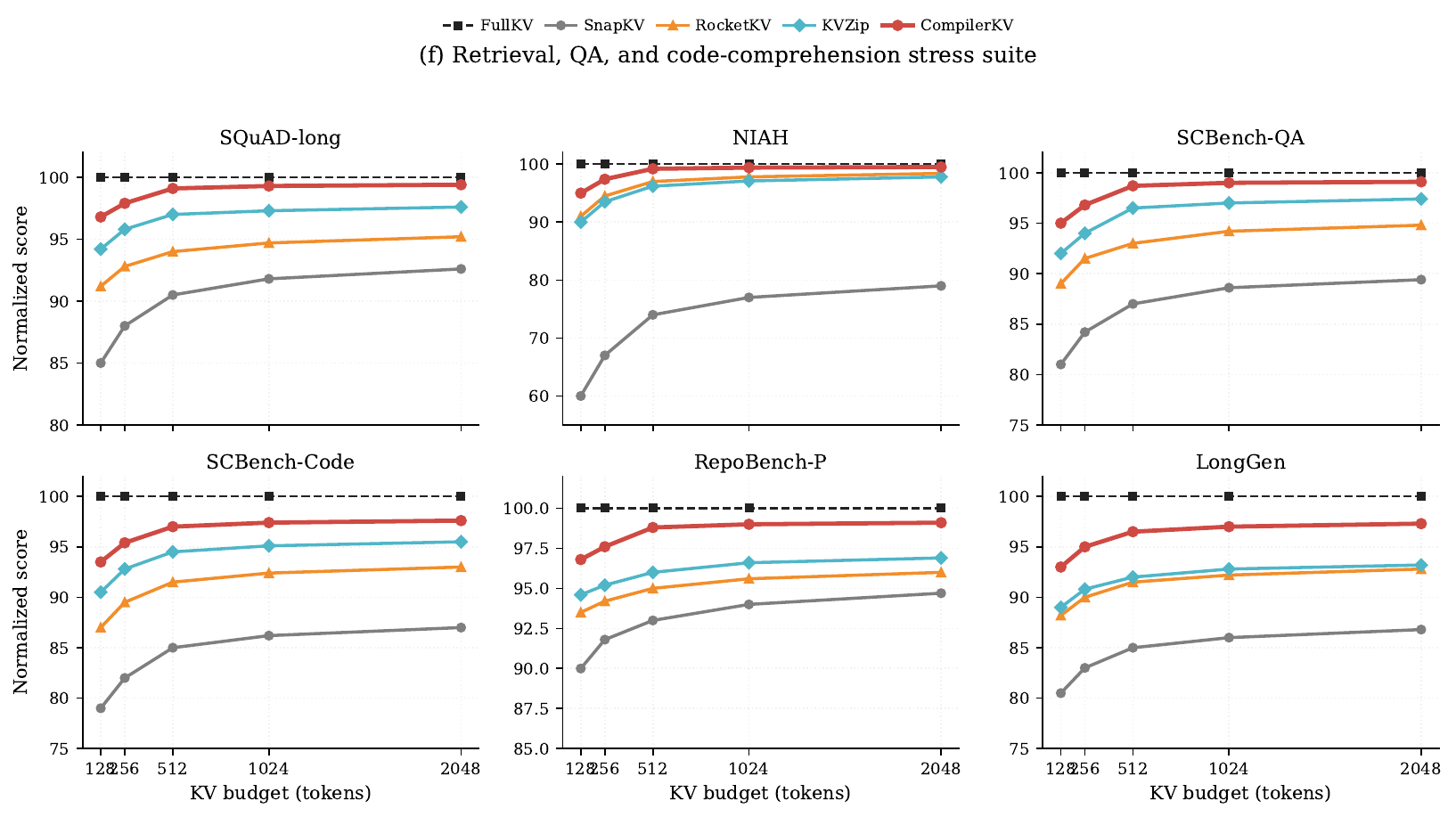}
\caption{\textbf{Retrieval, QA, and code-comprehension stress suite (LLaMA-3-8B-Instruct).} We show how each method scales with KV budget on SQuAD-long, NIAH, SCBench-QA, SCBench-Code, RepoBench-P, and LongGen. CompilerKV remains the strongest compressed method throughout the budget range and is especially robust in low-budget retrieval/code settings, which is consistent with the paper's view that offline compilation is a different decision paradigm rather than another online token scorer.}
\label{fig:retrieval_code_suite_extended}
\end{figure*}

\paragraph{Takeaway.} These plots show that the same budget-scaling trend appears across retrieval, reasoning, code, long-generation, and deployment stress settings.

\FloatBarrier
\clearpage
\subsection{Cross-benchmark Generalization (RULER)}

\begin{table}[H]
\centering
\caption{\textbf{RULER results} (16k context, Mistral-7B-Instruct-v0.2, 512-token budget). CompilerKV outperforms all prefill-only compression baselines, demonstrating generalization beyond LongBench.}
\setlength{\tabcolsep}{4pt}
\small
\begin{tabular}{lccccccc}
\toprule
\textbf{Method} & FullKV & StreamingLLM & H2O & SnapKV & PyramidKV & FastKV & \textbf{CompilerKV} \\
\midrule
Avg.\ Score & 95.0 & 15.0 & 27.1 & 75.6 & 66.5 & 77.8 & \textbf{81.4} \\
\bottomrule
\end{tabular}
\label{tab:ruler}
\end{table}

\begin{table}[H]
\centering
\caption{\textbf{Pressure-regime summary.} Stress settings where the gain is no longer a small LongBench-margin effect. RULER-128k uses the fixed $512/32k$ retained-cache ratio from Fig.~\ref{fig:ruler_length_scaling_extended}; NIAH is from Fig.~\ref{fig:niah_6_comparison}; batching is from Table~\ref{tab:batching_app}.}
\label{tab:pressure_summary}
\scriptsize
\setlength{\tabcolsep}{3pt}
\renewcommand{\arraystretch}{0.96}
\begin{tabularx}{0.96\textwidth}{lcccX}
\toprule
\textbf{Setting} & \textbf{FullKV} & \textbf{Best online baseline} & \textbf{CompilerKV} & \textbf{Takeaway} \\
\midrule
RULER-128k, fixed $512/32k$ ratio & $\sim$79.0 & $\sim$66.5 (KVZip) / $\sim$38.0 (SnapKV) & $\sim$73.0 & strongest compressed; below FullKV as expected \\
NIAH-32k retrieval & 0.92 & 0.42 (SnapKV) & 0.89 & $2.1\times$ SnapKV; preserves low-salience evidence \\
32k input, batch 16 & OOM & -- & $\sim$18.5 tok/s & feasible vs. infeasible; deployment break-even \\
\bottomrule
\end{tabularx}
\end{table}

\subsection{Cross-domain and Cross-backbone Stability of Head Reliability}
\label{sec:cross_domain}

\begin{table}[H]
\centering
\caption{\textbf{Cross-corpus stability of compiled head reliability across backbones.} For each backbone, we compile two independent Head Heterogeneity Tables using disjoint calibration corpora (arXiv$+$PubMed vs.\ ShareGPT$+$UltraChat) and report Spearman $\rho$ between the resulting per-head reliability rankings. The stable correlations across all four architectures support the claim that the compiled table captures model-level head functionality rather than a corpus artifact.}
\setlength{\tabcolsep}{8pt}
\small
\begin{tabular}{lccc}
\toprule
\textbf{Backbone} & \textbf{Spearman $\rho$} & \textbf{Layer-bin mean} & \textbf{Layer-bin min} \\
\midrule
LLaMA-3-8B        & 0.91 & 0.91 & 0.88 \\
Mistral-7B        & 0.90 & 0.90 & 0.87 \\
Qwen2-7B          & 0.90 & 0.89 & 0.86 \\
InternLM2.5-7B    & 0.89 & 0.89 & 0.85 \\
\midrule
\textbf{Mean}     & \textbf{0.90} & \textbf{0.90} & \textbf{0.87} \\
\bottomrule
\end{tabular}
\label{tab:cross_domain_app}
\end{table}

\subsection{Compile Once, Deploy Across Models}
\label{sec:model_transfer}
\begin{table}[H]
\centering
\caption{\textbf{Model-to-model transfer of compiled retention tables.} We compile on a source backbone, deploy the resulting table on a target backbone without target recompilation, and compare with target-compiled CompilerKV. Same-family transfer costs $0.3$--$0.4$ points; broader cross-family transfer costs $0.4$--$0.8$ points on average, showing that the compiled table is a portable retention prior rather than a per-run tuning artifact.}
\setlength{\tabcolsep}{4pt}
\small
\begin{tabular}{llccc}
\toprule
\textbf{Source table} & \textbf{Target model} & \textbf{Target-compiled} & \textbf{Transferred} & \textbf{Drop} \\
\midrule
\multicolumn{5}{l}{\textit{Same-family transfer}} \\
LLaMA-3-8B     & LLaMA-3.1-8B      & 52.60 & 52.30 & $-0.30$ \\
LLaMA-3.1-8B   & LLaMA-3-8B        & 42.55 & 42.28 & $-0.27$ \\
LLaMA-3.1-8B   & LLaMA-3.2-family  & 50.10 & 49.70 & $-0.40$ \\
\midrule
\multicolumn{5}{l}{\textit{Cross-family transfer}} \\
Mistral-7B     & LLaMA-3-8B        & 42.55 & 41.85 & $-0.70$ \\
LLaMA-3-8B     & Mistral-7B        & 42.94 & 42.16 & $-0.78$ \\
Qwen2-7B       & InternLM-2.5      & 42.61 & 41.98 & $-0.63$ \\
InternLM-2.5   & Qwen2-7B          & 41.13 & 40.44 & $-0.69$ \\
\midrule
\textbf{Mean drop} & -- & -- & -- & $\mathbf{-0.54}$ \\
\bottomrule
\end{tabular}
\label{tab:model_transfer}
\end{table}

\subsection{Batching Scalability}
\label{sec:batching_app}

\begin{table}[H]
\centering
\caption{\textbf{Main efficiency profile.} Retained prefill KV ratio and decoding efficiency on a single A100-80GB. CompilerKV fixes the prefill budget to 512 tokens; generated KV states are appended normally. Speedup multipliers in parentheses are relative to FullKV at the matching context length.}
\label{tab:efficiency_compilerkv}
\scriptsize
\setlength{\tabcolsep}{3pt}
\renewcommand{\arraystretch}{0.95}
\resizebox{\textwidth}{!}{%
\begin{tabular}{cc l ccc ccc}
\toprule
\multicolumn{2}{c}{\textbf{Context}} & \multirow{2}{*}{\textbf{Method}}
& \multicolumn{3}{c}{\textbf{Prefill KV Compression}}
& \multicolumn{3}{c}{\textbf{Decoding Efficiency}} \\
\cmidrule(lr){1-2} \cmidrule(lr){4-6} \cmidrule(lr){7-9}
\textbf{Input} & \textbf{Output} && \textbf{Budget} & \textbf{Retained} & \textbf{Saving}
& \textbf{TTFT(s)$\downarrow$} & \textbf{Decode rate(tok/s)$\uparrow$} & \textbf{Latency(s)$\downarrow$} \\
\midrule
\multirow{2}{*}{8k}  & \multirow{2}{*}{2k} & FullKV     & 8k  & 100.0\% & 0.0\%  & \textbf{0.66} & 27.63 & 74.79 \\
                    &                     & CompilerKV & 512 & \textbf{6.25\%} & \textbf{93.75\%} & 0.72 & \textbf{36.20} {\scriptsize($1.31\times$)} & \textbf{58.90} {\scriptsize($1.27\times$)} \\
\midrule
\multirow{2}{*}{16k} & \multirow{2}{*}{4k} & FullKV     & 16k & 100.0\% & 0.0\%  & \textbf{1.45} & 19.55 & 209.56 \\
                    &                     & CompilerKV & 512 & \textbf{3.12\%} & \textbf{96.88\%} & 1.52 & \textbf{34.80} {\scriptsize($1.78\times$)} & \textbf{118.40} {\scriptsize($1.77\times$)} \\
\midrule
\multirow{2}{*}{32k} & \multirow{2}{*}{8k} & FullKV     & 32k & 100.0\% & 0.0\%  & \textbf{3.52} & 11.65 & 706.56 \\
                    &                     & CompilerKV & 512 & \textbf{1.56\%} & \textbf{98.44\%} & 3.65 & \textbf{28.50} {\scriptsize($2.45\times$)} & \textbf{285.00} {\scriptsize($2.48\times$)} \\
\bottomrule
\end{tabular}}
\end{table}

\begin{table}[H]
\centering
\caption{\textbf{Batching scalability on A100-80GB} (LLaMA-3-8B, 32k input + 8k output). Memory denotes peak GPU memory including weights, runtime buffers, and KV states. FullKV becomes infeasible at batch size 16 while CompilerKV continues to serve, with up to $4.4\times$ decoding throughput.}
\setlength{\tabcolsep}{6pt}
\small
\begin{adjustbox}{max width=\textwidth}
\begin{tabular}{c|ccc|c|ccc}
\toprule
\textbf{Batch} & \multicolumn{3}{c|}{\textbf{Peak GPU memory}} & \textbf{Fits A100-80GB?} & \multicolumn{3}{c}{\textbf{Decoding Throughput (tok/s)}} \\
& FullKV & CompilerKV & Saved & FullKV / CompilerKV & FullKV & CompilerKV & Speedup \\
\midrule
1  & $\sim$19 GB & $\sim$15 GB & 3.94 GB  & \cmark / \cmark & 11.65 & 28.50 & $2.45\times$ \\
4  & $\sim$34 GB & $\sim$18 GB & 15.75 GB & \cmark / \cmark & $\sim$7.4 & $\sim$25.5 & $\sim$3.4$\times$ \\
8  & $\sim$54 GB & $\sim$23 GB & 31.50 GB & \cmark / \cmark & $\sim$5.1  & $\sim$22.3 & $\sim$4.4$\times$ \\
16 & $\sim$94 GB & $\sim$31 GB & 63.00 GB & \xmark / \cmark & OOM    & $\sim$18.5 & -     \\
\bottomrule
\end{tabular}
\end{adjustbox}
\label{tab:batching_app}
\end{table}

\subsection{Needle-in-a-Haystack Pressure Test}

\begin{figure}[H]
    \centering
    \begin{subfigure}[b]{0.495\linewidth}
        \centering
        \includegraphics[width=\linewidth]{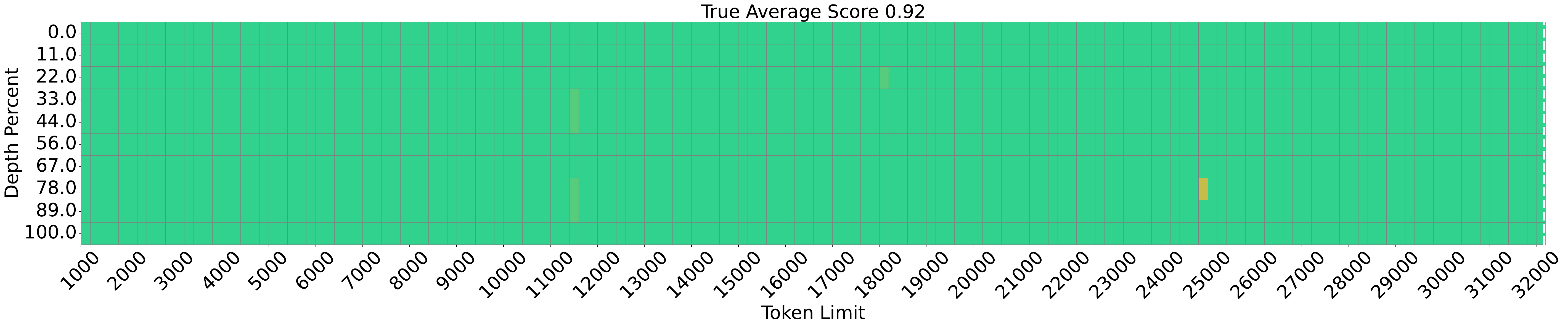}
        \caption{FullKV (Avg. Score: 0.92)}
        \label{fig:niah_fullkv}
    \end{subfigure}
    \hfill
    \begin{subfigure}[b]{0.495\linewidth}
        \centering
        \includegraphics[width=\linewidth]{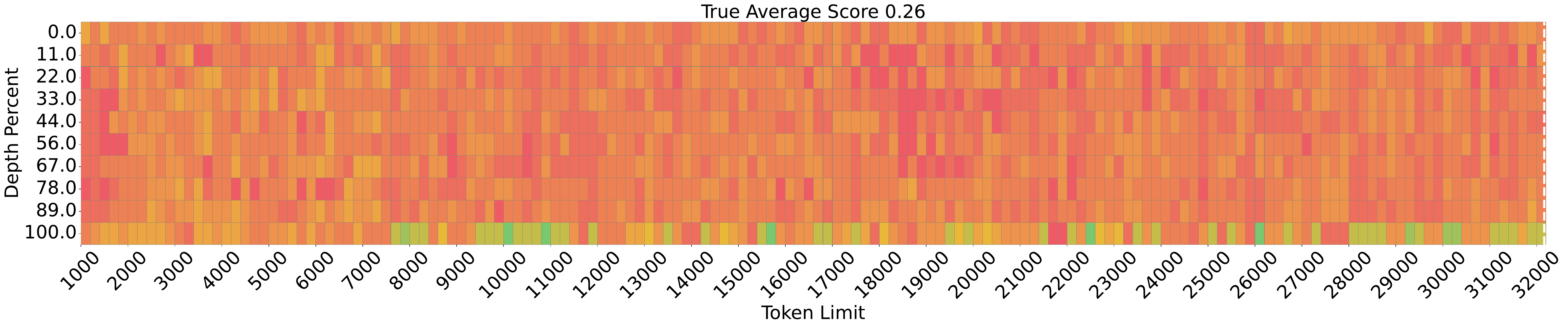}
        \caption{StreamingLLM (Avg. Score: 0.26)}
        \label{fig:niah_streaming}
    \end{subfigure}
    \vspace{0.2cm}
    \begin{subfigure}[b]{0.495\linewidth}
        \centering
        \includegraphics[width=\linewidth]{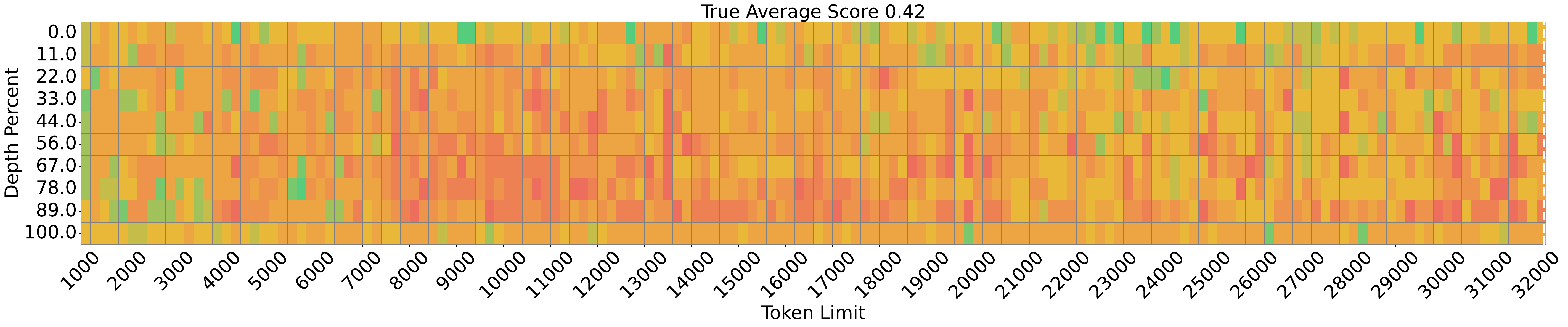}
        \caption{SnapKV (Avg. Score: 0.42)}
        \label{fig:niah_snapkv}
    \end{subfigure}
    \hfill
    \begin{subfigure}[b]{0.495\linewidth}
        \centering
        \includegraphics[width=\linewidth]{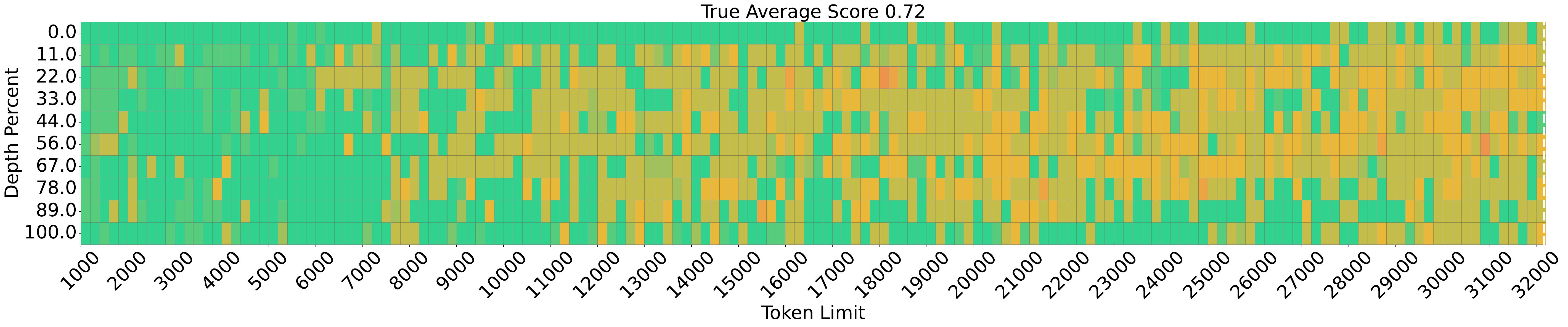}
        \caption{PyramidKV (Avg. Score: 0.72)}
        \label{fig:niah_pyramidkv}
    \end{subfigure}
    \vspace{0.2cm}
    \begin{subfigure}[b]{0.495\linewidth}
        \centering
        \includegraphics[width=\linewidth]{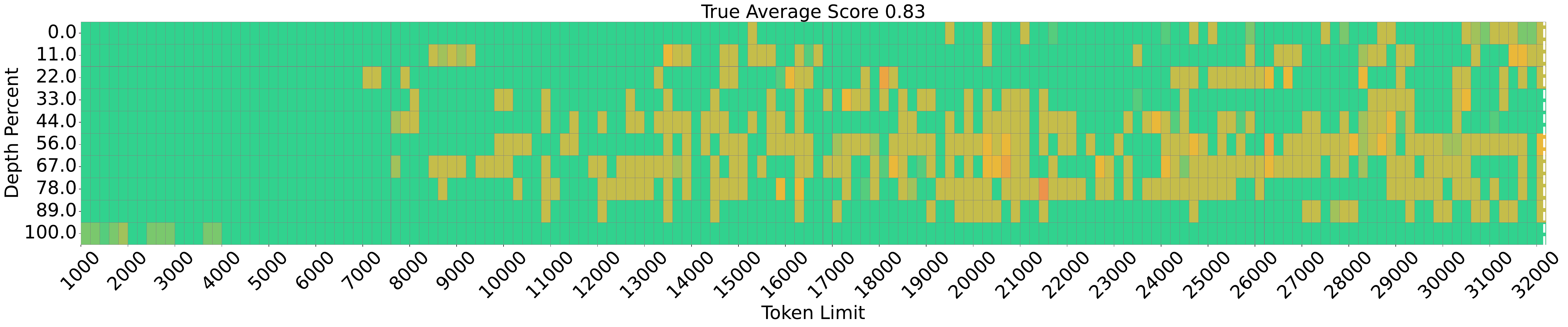}
        \caption{AdaKV (Avg. Score: 0.83)}
        \label{fig:niah_dynamickv}
    \end{subfigure}
    \hfill
    \begin{subfigure}[b]{0.495\linewidth}
        \centering
        \includegraphics[width=\linewidth]{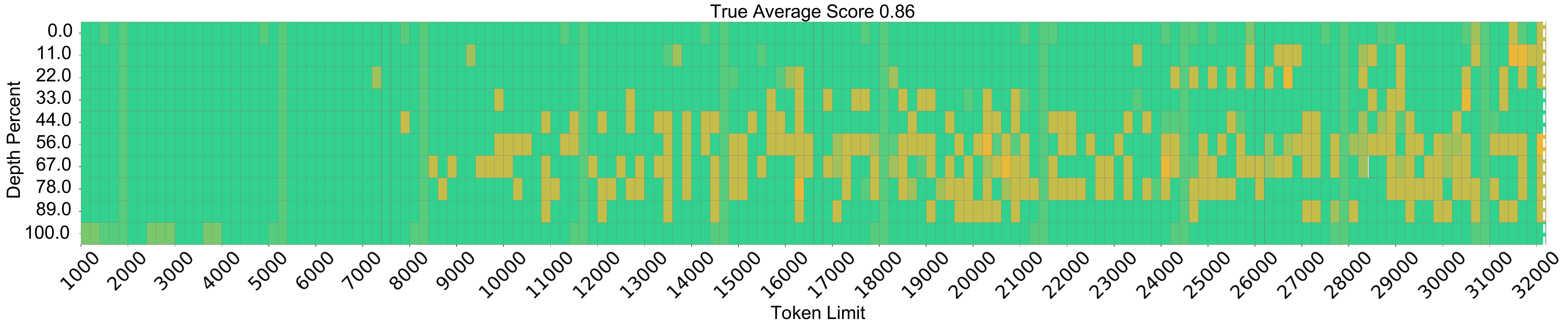}
        \caption{\textbf{CompilerKV (Avg. Score: 0.89) — Ours}}
        \label{fig:niah_compilerkv}
    \end{subfigure}
    \caption{\textbf{Needle-in-a-Haystack Pressure Test on Mistral-7B.} Visual comparison of retrieval accuracy (Green=100\%, Red=0\%) across varying context lengths and needle depths. While baselines like StreamingLLM and SnapKV struggle with long-range dependencies and AdaKV shows fragmentation at extreme lengths, CompilerKV maintains a robust retrieval pattern comparable to the FullKV oracle.}
    \label{fig:niah_6_comparison}
\end{figure}

\subsection{Offline Table Compilation}
\label{sec:offline_table_compilation}

\paragraph{CQL conservative-weight sensitivity.}
Because our problem is a horizon-1 contextual bandit, the algorithmic difference between CQL and plain reward regression is the support-aware conservative penalty, not multi-step bootstrapping. We therefore sweep the conservative weight $\alpha$: $\alpha=0$ reduces to plain regression and over-selects unsupported aggressive thresholds, while overly large $\alpha$ becomes too conservative and approaches a fixed high-retention rule. The best range ($\alpha\in[0.5,1.0]$) consistently matches the main-table setting, confirming that CQL's gain comes from calibrated pessimism in rare high-risk states.

\paragraph{Calibration corpus.}
We compile both decision tables (the head heterogeneity table $W_{\text{head}}[l,h]$ and the risk-adaptive threshold gating  $\bf{T}_\text{gate}$) on a held-out \emph{calibration corpus} $\mathcal{D}_{\text{cal}}$, which is strictly disjoint from the LongBench evaluation set to avoid leakage.
$\mathcal{D}_{\text{cal}}$ contains approximately $50$K long-context prompts sampled from diverse public sources, including long-form narratives (PG19), scientific/technical articles (arXiv, PubMed), long-document summarization and meeting transcripts (GovReport, QMSum, BookSum), as well as instruction-style dialogues and code/QA text (ShareGPT/UltraChat and The Pile subsets such as GitHub/StackExchange).
All prompts are used in an unlabeled manner: we only require forward statistics from the prefill stage and a short continuation loss for reward computation.

\paragraph{Risk signals and stratified coverage.}
For each prompt $x_{1:T}$, we compute the observation-window statistics on $\mathcal{W}=\{T-w_{\text{obs}}+1,\dots,T\}$, including the attention entropy $\mathcal{R}_{\text{struct}}$ and local perplexity $\mathcal{R}_{\text{sem}}$, and discretize them into bin indices $(b_h,b_p)$.
To ensure sufficient coverage of high-risk prompts (e.g., high-entropy or high-perplexity regimes), we adopt stratified sampling over $(b_h,b_p)$ when constructing $\mathcal{D}_{\text{cal}}$, avoiding a dominance of medium-risk samples that would otherwise bias the learned tables toward overly aggressive pruning.

\paragraph{Why $20\times4$ bins for $(H(A),\mathrm{PPL})$.}
We choose $N_H{=}20$ entropy bins and $N_P{=}4$ perplexity bins, resulting in a $20\times4$ risk grid for each layer.
This design is motivated by a practical bias--variance trade-off for discrete table policies:
(i) attention entropy is a \emph{structural} signal with a relatively smooth and wide dynamic range over long-context prompts, hence we allocate finer granularity ($20$ bins) to capture gradual shifts in concentration vs.\ dispersion;
(ii) perplexity acts as a \emph{semantic uncertainty} signal whose main role is to modulate the overall conservativeness, and empirically exhibits heavier tails and higher estimation noise on short windows, thus we use coarser partitioning ($4$ bins) to maintain stability.
With $L{=}32$ layers, the LUT contains $32\times20\times4{=}2560$ entries, which is small enough for robust offline estimation and deployment, while still expressive enough to differentiate risk regimes without introducing over-fragmentation.
In our calibration, this binning yields adequate per-cell support under stratified sampling and avoids sparse high-risk cells that would cause unstable threshold decisions.

\paragraph{Compilation hyperparameters.}
Unless otherwise stated, the observation window is $w_{\mathrm{obs}}=64$ query positions and the held-out continuation window for reward evaluation is 128 tokens. We train the horizon-1 CQL tables with AdamW, learning rate $3\times10^{-4}$, batch size 4096 state-action records, 10 epochs over the calibration set, conservative weight $\alpha=0.75$, threshold grid $\tau\in[0.8,1.0]$ with 21 uniformly spaced actions, and head-weight grid $\mathcal{W}=[0.8,1.5]$ with 29 uniformly spaced actions. The risk grid uses $20\times4$ entropy--PPL bins, and the budget penalty is $\lambda=\beta_{\mathrm{bud}}/T$ with $\beta_{\mathrm{bud}}=1.0$ for the gate and $\lambda=0$ for the head table.

% \paragraph{Reward construction and rollout length.}
% For table compilation, we use a short continuation window $\mathcal{W}'$ (typically $|\mathcal{W}'|{=}128$ tokens) to estimate the compression-induced loss increase.
% Concretely, given an action (a head weight in Stage 2 or a layer threshold in Stage 3), we run the full prefill-only compression pipeline and compute the NLL gap
% $\Delta\mathcal{L}=\mathcal{L}_{\text{comp}}-\mathcal{L}_{\text{full}}$ on $\mathcal{W}'$.
% This short-horizon evaluation is sufficient to provide consistent preference signals for one-shot prefill decisions, while keeping compilation cost manageable.

\paragraph{Model choice and cross-architecture applicability.}
We compile tables on a 32-layer backbone to obtain a unified layer index space.
For models with different depths, we apply a monotonic depth mapping (e.g., relative-depth interpolation) at inference time.
While our main results use per-model compilation, \S\ref{sec:model_transfer} reports direct source-to-target table transfer across both closely related LLaMA variants and broader 7B--8B backbone pairs. More distant architecture families and larger-scale transfer remain future work.

\paragraph{Hyperparameter ranges and robustness.}
Both compiled tables operate as \emph{bounded multiplicative modulators} on top of a normalized utility signal, which allows us to use conservative and architecture-stable ranges.
For Stage 3, the LUT outputs a layer threshold $\tau^{(l)}\in[0.8,1.0]$.
This range is chosen for two reasons.
First, $\hat{u}^{(l)}_t$ is constructed from $u_t=\alpha_t\cdot\rho_t$, where $\alpha_t=(T/W)\sum_{j\in\Omega}\bar A_{j,t}$ has prompt-level mean scale near one and $\rho_t$ is row-normalized by the mean value norm. Thus $\hat{u}^{(l)}_t$ concentrates around $\mathcal{O}(1)$ across prompts and layers, so thresholds outside $[0.8,1.0]$ either admit almost all tokens or reject nearly all tokens, collapsing the ``threshold+Top-$B_l$'' mechanism.
Second, under tight budgets the final retention is always enforced by Top-$B_l$ ; hence $\tau^{(l)}$ mainly controls the \emph{candidate set margin} rather than the final budget.
Bounding $\tau^{(l)}$ to $[0.8,1.0]$ keeps the candidate size within a stable, budget-aligned regime while still allowing risk-conditioned shifts across $(b_h,b_p)$.

\begin{figure}[h]
\centering
\includegraphics[width=0.58\textwidth]{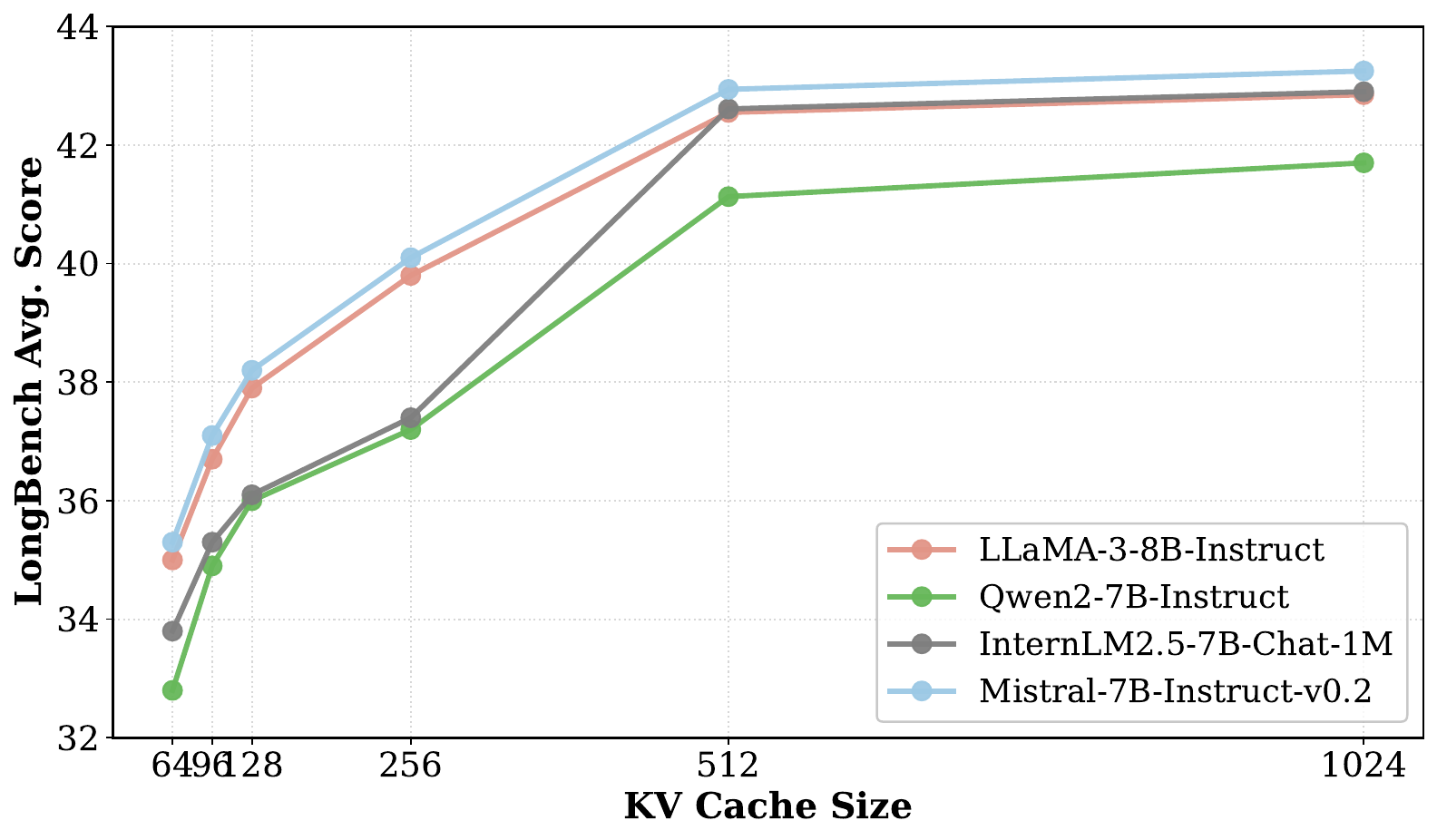}
\caption{\textbf{Universality across model architectures.} Average performance trends on four different LLMs (LLaMA-3, Qwen2, InternLM2.5, and Mistral) under varying KV budgets. The consistent degradation patterns support that CompilerKV's risk-adaptive mechanism generalizes across architectures.}
\label{fig:model_universality}
\end{figure}

\textbf{Universality across Model Architectures.} Figure \ref{fig:model_universality} demonstrates our method's consistent performance trajectory across four distinct LLMs (including LLaMA-3, Qwen2, and Mistral) as the budget tightens from 1024 to 64 tokens. This architectural universality confirms that our Risk-Adaptive Gating and Head Reliability mechanisms capture fundamental attention properties rather than overfitting to specific model weights.

% \begin{figure}[h]
% \centering
% \includegraphics[width=0.48\textwidth]{kv_cache_size_avg_score.pdf} 
% \caption{\textbf{Universality across Model Architectures.} The average performance trend of our method on four different LLMs (LLaMA-3, Qwen2, InternLM-2.5, Mistral) under varying KV cache budgets (ranging from 64 to 1024). The consistent degradation patterns across diverse architectures demonstrate the generalization capability of our risk-adaptive mechanism.}
% \label{fig:model_universality}
% \end{figure}

\subsection{Visualization of Learned Policies}

To understand how CompilerKV adapts to different risk levels, we visualize the learned threshold table $\mathbf{T}_\text{gate}$ in Figure \ref{fig:ppl_bins}. The visualization confirms that the offline reinforcement learning process converges to an interpretable strategy:
As shown in the difference between the ``Low Perplexity" and ``High Perplexity" subplots, the policy globally lowers the retention threshold (blue zones) for high-perplexity samples, allowing more tokens to be stored when the model is uncertain.  Within each plot, as Attention Entropy increases ($x$-axis), the threshold decreases. This indicates the model automatically acts conservatively for diffuse attention patterns (e.g., dense retrieval), preventing the pruning of scattered but critical information.

For head reliability, we use two consistent pieces of evidence: (i) the aggregated reliability distribution in Figure~\ref{fig:e1e2e3_diagnostic}(b), which tests whether compiled weights separate noisy and retrieval-critical populations, and (ii) the cross-backbone stability table above, which tests whether the ranking transfers across corpora and architectures. We therefore do not rely on a separate layer--head heatmap as statistical evidence for bimodality; spatial heatmaps answer a different question---where reliable heads occur---and can look continuous even when the pooled reliability distribution is mixture-like.

\begin{figure*}[!ht]
\centering
\includegraphics[width=0.95\textwidth]{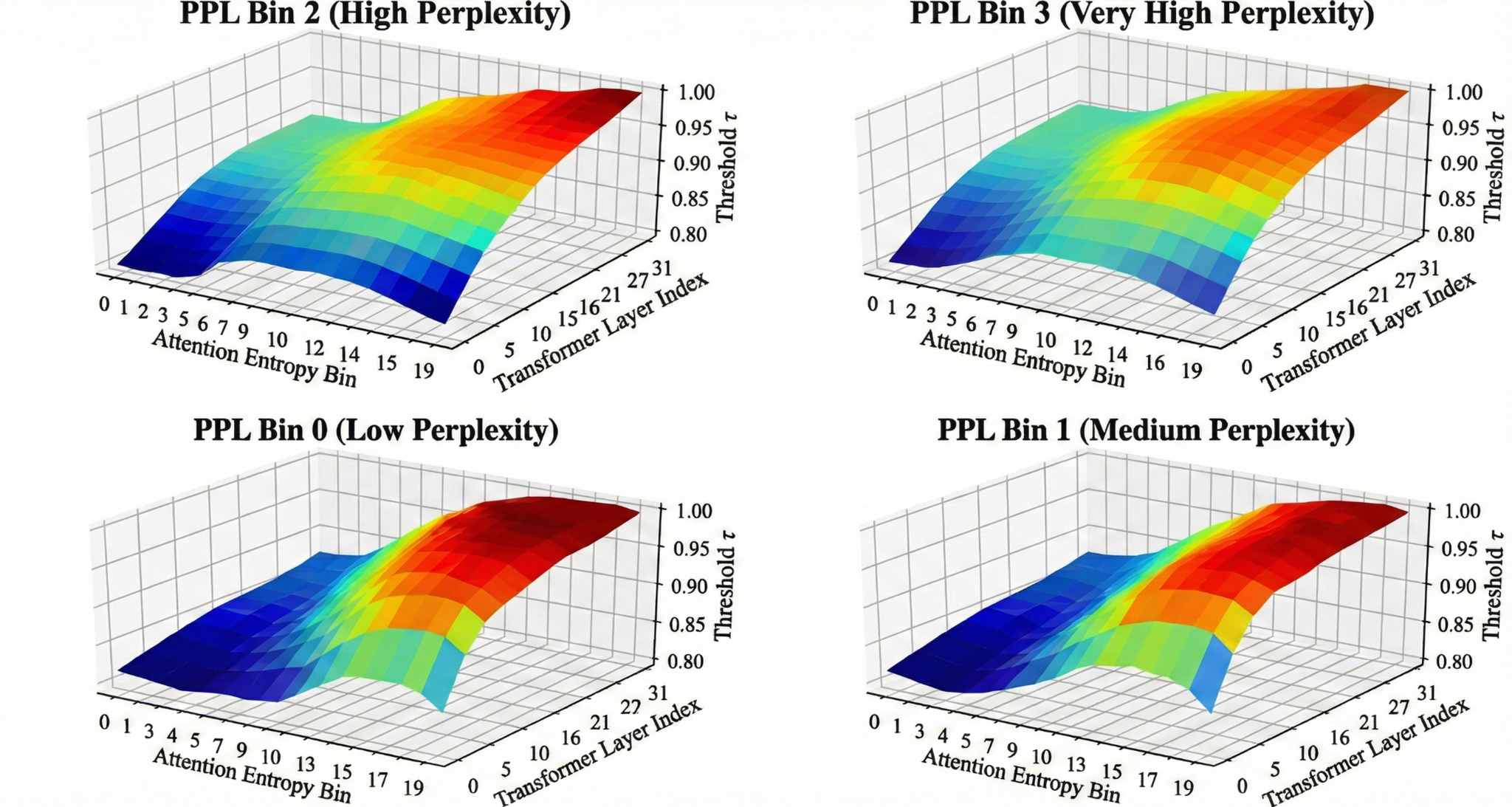} % Placeholder
\caption{\textbf{Risk-Adaptive Threshold Gating Policy.} The plots show the learned retention threshold ($\tau$) across layers for samples with varying risk levels (binned by Perplexity and Attention Entropy). For high-risk samples (High Perplexity, rightmost plots), the policy automatically lowers thresholds to preserve more information, validating our risk-adaptive mechanism.}
\label{fig:ppl_bins}
\end{figure*}

\begin{table*}[h]
\centering
\caption{\textbf{LongBench results at the extreme 128-token budget.} CompilerKV remains the strongest compressed method on all four backbones, but the severe information bottleneck keeps it below the FullKV reference, as expected. Bold marks the best compressed average; FullKV is shown only as the lossless reference.}
\label{table:longbench_128}
\resizebox{\textwidth}{!}{
\begin{tabular}{l@{\hspace{0.05ex}}|c@{\hspace{0.05ex}}c@{\hspace{0.05ex}}c@{\hspace{0.05ex}}c@{\hspace{0.05ex}}c@{\hspace{0.05ex}}c@{\hspace{0.05ex}}c@{\hspace{0.05ex}}c@{\hspace{0.05ex}}c@{\hspace{0.05ex}}c@{\hspace{0.05ex}}c@{\hspace{0.05ex}}c@{\hspace{0.05ex}}c@{\hspace{0.05ex}}c@{\hspace{0.05ex}}c@{\hspace{0.6ex}}c@{\hspace{0.6ex}}c@{\hspace{0.6ex}}c}

%\specialrule{1pt}{0pt}{2pt}
\toprule

\multirow{3}{*}{\rotatebox{30}{\textbf{Model}}}  & \multirow{3}{*}{\rotatebox{30}{\textbf{Method}}}  & \multicolumn{3}{c}{\itbf{Single-Document QA}} & \multicolumn{3}{c}{\itbf{Multi-Document QA}}& \multicolumn{3}{c}{\itbf{Summarization}}& \multicolumn{3}{c}{\itbf{Few-shot Learning}}& \multicolumn{2}{c}{\itbf{Synthetic}} & \multicolumn{2}{c}{\itbf{Code}} & \multirow{3}{*}{\itbf{Avg.}} \\

\cmidrule(lr){3-5}\cmidrule(lr){6-8}\cmidrule(lr){9-11}\cmidrule(lr){12-14}\cmidrule(lr){15-16}\cmidrule(lr){17-18}

 & & \rotatebox[origin=c]{30}{\textbf{NrtvQA}} & \rotatebox[origin=c]{30}{\textbf{Qasper}} & \rotatebox[origin=c]{30}{\textbf{MF-en}} & \rotatebox[origin=c]{30}{\textbf{HotpotQA}} & \rotatebox[origin=c]{30}{\textbf{2WikiMQA}} & \rotatebox[origin=c]{30}{\textbf{Musique}} & \rotatebox[origin=c]{30}{\textbf{GovReport}} & \rotatebox[origin=c]{30}{\textbf{QMSum}} & \rotatebox[origin=c]{30}{\textbf{MultiNews}} & \rotatebox[origin=c]{30}{\textbf{TREC}} & \rotatebox[origin=c]{30}{\textbf{TriviaQA}} & \rotatebox[origin=c]{30}{\textbf{SAMSum}} & \rotatebox[origin=c]{30}{\textbf{PCount}} & \rotatebox[origin=c]{30}{\textbf{PRe}} & \rotatebox[origin=c]{30}{\textbf{Lcc}} & \rotatebox[origin=c]{30}{\textbf{RB-P}} & \\
\cmidrule(lr){3-18}

 & &18409 &3619 &4559 &9151 &4887 &11214 &8734 &10614 &2113 &5177 &8209 &6258 &11141 &9289 &1235 &4206& -- \\

\midrule
% llama3-8b-instruct 128
\multirow{7}{*}{\rotatebox{90}{\textbf{\makecell{LLaMA-3-8B\\-Instruct}}}}
 &FullKV &25.16 &31.81 &39.59 &43.09 &36.15 &21.77 &28.62 &23.34 &26.33 &75.00 &90.50 &42.36 &5.20 &69.25 &59.04 &53.93 & 41.95 \\
% llama3-8b-instruct 128
% llama3-8b-instruct 512
\cmidrule(lr){2-19}
&StreamingLLM &17.85 &9.50 &23.09 &37.84 &29.02 &16.77 &17.91 &20.42 &20.16 &44.00 &73.00 &30.00 &5.80 &69.50 &48.38 &49.31 & 32.03 \\
&H2O &21.58 &12.54 &28.49 &37.13 &32.36 &18.88 &20.23 &22.16 &21.14 &39.00 &86.62 &39.19 &5.50 &69.50 &57.39 &54.46 & 35.39 \\
&SnapKV &21.71 &12.37 &32.38 &37.44 &30.48 &19.50 &19.06 &21.36 &20.07 &45.5 &87.74 &38.15 &5.50 &68.85 &57.42 &54.61 & 35.76 \\
&PyramidKV &22.26 &16.65 &30.73 &38.97 &29.28 &19.19 &19.92 &22.06 &20.87 &68.00 &88.95 &38.23 &5.92 &69.50 &57.20 &51.54 & 37.45 \\
&DynamicKV &22.10 &14.93 &32.94 &41.06 &27.98 &21.18 &20.03 &22.06 &21.28 &65.50 &89.61 &38.70 &5.13 &69.50 &58.01 &54.00 & 37.75 \\
&CompilerKV &24.46 &28.22 &36.34 &42.37 &30.48 &21.25 &27.43 &22.16 &22.35 &72.5 &89.68 &41.00 &6.50 &69.25 &58.91 &53.94 & \textbf{40.43} \\
\midrule

\multirow{7}{*}{\rotatebox{90}{\textbf{\makecell{Mistral-7B\\-Instruct-v0.2}}}}
 &FullKV &26.63 &32.99 &49.34 &42.77 &27.35 &18.77 &32.87 &24.24 &27.10 &71.00 &86.23 &42.96 &2.75 &86.98 &56.93 &54.49 & 42.71 \\
% mistral-7b-instruct-v0.2 128
% mistral-7b-instruct-v0.2 512
\cmidrule(lr){2-19}
&StreamingLLM &16.58 &14.76 &30.36 &28.13 &21.76 &11.98 &18.26 &19.02 &19.16 &43.50 &74.12 &28.50 &2.50 &31.81 &43.65 &41.19 & 27.83 \\
&H2O &21.66 &21.64 &38.60 &30.96 &20.63 &13.02 &20.65 &22.61 &22.08 &39.00 &82.19 &39.75 &3.16 &79.98 &51.25 &48.20 & 34.71 \\
&SnapKV &20.11 &21.28 &42.98 &37.51 &22.31 &14.43 &19.19 &21.89 &21.01 &48.00 &83.77 &40.44 &2.51 &66.99 &51.64 &48.57 & 35.16 \\
&PyramidKV &22.11 &22.52 &43.04 &33.57 &22.98 &15.69 &20.56 &22.52 &21.36 &65.50 &83.84 &40.03 &2.89 &67.26 &51.51 &46.42 & 36.36 \\
&DynamicKV &22.05 &23.65 &43.08 &36.03 &22.60 &15.23 &21.35 &23.11 &22.19 &68.00 &84.79 &41.02 &4.20 &70.11 &52.45 &47.41 & 37.33 \\
&CompilerKV &25.31 &29.06 &46.20 &41.42 &25.91 &17.47 &28.55 &23.21 &26.35 &69.20 &85.00 &41.20 &4.25 &80.30 &54.30 &51.60 & \textbf{40.58} \\
\midrule

\multirow{7}{*}{\rotatebox{90}{\textbf{\makecell{Qwen2-7B\\-Instruct}}}}
 &FullKV &25.14 &42.35 &45.04 &14.80 &14.13 &9.23 &36.35 &23.79 &26.51 &76.50 &89.16 &45.23 &6.50 &75.50 &60.30 &60.78 & 40.71 \\
% qwen2-7b-instruct 128
% qwen2-7b-instruct 512
\cmidrule(lr){2-19}
&StreamingLLM &19.25 &23.63 &26.51 &14.00 &15.30 &7.46 &18.07 &19.30 &18.30 &47.00 &77.92 &31.57 &6.50 &17.00 &42.52 &41.94 & 26.64 \\
&H2O &20.33 &30.43 &34.22 &13.61 &13.37 &7.81 &20.72 &21.66 &18.44 &40.00 &86.94 &42.17 &7.00 &70.50 &53.45 &53.76 & 33.40 \\
&SnapKV &22.26 &31.62 &38.95 &16.05 &17.71 &7.66 &18.91 &21.41 &18.21 &46.00 &87.61 &42.01 &6.50 &63.50 &54.87 &53.03 & 34.14 \\
&PyramidKV &20.50 &31.70 &39.95 &18.54 &18.54 &8.85 &19.24 &20.47 &18.18 &60.00 &87.98 &39.71 &7.00 &49.00 &48.77 &47.91 & 33.52 \\
&DynamicKV &22.77 &35.57 &42.62 &14.80 &16.35 &8.31 &21.41 &21.97 &19.56 &58.00 &88.18 &40.93 &6.50 &70.00 &53.58 &52.50 & 35.82 \\
&CompilerKV &24.12 &38.93 &43.81 &14.85 &16.45 &8.40 &34.62 &22.47 &23.74 &70.00 &88.38 &44.05 &6.50 &72.50 &56.44 &57.79 & \textbf{38.94} \\
\midrule

\multirow{7}{*}{\rotatebox{90}{\textbf{\makecell{InternLM-2.5\\-7B-Chat-1M}}}}
 &FullKV &22.42 &27.61 &39.98 &40.92 &33.48 &26.68 &33.01 &25.18 &26.28 &72.50 &86.76 &39.76 &2.91 &100.00 &55.86 &57.95 & 43.21 \\
% InternLM-2.5-7B-Chat-1M 128
% InternLM-2.5-7B-Chat-1M 512
\cmidrule(lr){2-19}
&StreamingLLM &17.91 &13.02 &24.31 &24.27 &16.01 &11.29 &17.29 &20.62 &18.06 &48.5 &67.53 &21.93 &0.82 &87.39 &43.45 &42.79 & 29.70 \\
&H2O &16.16 &17.71 &27.94 &26.83 &17.83 &17.81 &13.99 &22.59 &16.9 &39.50 &81.87 &32.15 &1.32 &96.50 &48.30 &47.27 & 32.79 \\
&SnapKV &19.65 &17.44 &35.29 &27.36 &18.58 &19.79 &12.76 &22.42 &16.31 &48.00 &80.23 &31.35 &0.95 &95.00 &49.47 &48.22 & 33.93 \\
&PyramidKV &18.80 &17.35 &33.48 &31.16 &20.05 &19.02 &14.65 &22.02 &17.40 &69.50 &80.87 &32.02 &1.23 &95.00 &47.13 &44.73 & 35.28 \\
&DynamicKV &17.93 &19.89 &34.15 &31.50 &19.03 &20.60 &15.14 &22.41 &18.15 &70.00 &83.09 &32.44 &0.86 &95.50 &49.33 &47.16 & 36.07 \\
&CompilerKV &21.36 &25.32 &36.61 &37.17 &29.75 &24.95 &25.5 &22.62 &18.46 &70.50 &83.43 &36.76 &1.92 &98.00 &54.28 &55.42 & \textbf{40.13} \\
\bottomrule
\end{tabular}
}
\vspace{-3mm}
\end{table*}

\section{Algorithm}

% This appendix algorithm is flushed before the NeurIPS checklist by the \clearpage below.
\begin{algorithm}[tb]
\caption{\ours: Prefill-only KV Compression (operator form)}
\label{alg:compilerkv}
\begin{algorithmic}[1]
\STATE \textbf{Input:} prompt $x_{1:T}$; layers $l\in[1,L]$; heads $h\in[1,H]$; per-layer budgets $\{B_l\}$;
observation window $\Omega=\{T-w_{\mathrm{obs}}+1,\dots,T\}$;
risk LUT $\mathbf{T}_{\mathrm{gate}}$; head table $\mathbf{W}_{\mathrm{head}}$
\STATE \textbf{Output:} compressed KV cache $\{(\tilde{K}^{(l)},\tilde{V}^{(l)})\}_{l=1}^{L}$

\vspace{0.35em}
\STATE \textbf{(A) Prefill-time risk coordinates (streaming reductions)}
\STATE $\{A^{(l,h)},\,K^{(l,h)},\,V^{(l,h)}\}\leftarrow \textsc{Prefill}(x_{1:T})$
\hfill $A^{(l,h)}\!\in\!\mathbb{R}^{T\times T},\ K^{(l,h)},V^{(l,h)}\!\in\!\mathbb{R}^{T\times d}$
\STATE $\bar{A}\leftarrow \textsc{Mean}_{l,h}\!\left(A^{(l,h)}\right)$ \hfill (\textit{layer/head mean})
\STATE $\alpha \leftarrow (T/|\Omega|)\sum_{j\in\Omega} \bar{A}_{j,:}\in\mathbb{R}^{T}$ \hfill (\textit{scale-normalized window mass})
\STATE $\bar{A}' \leftarrow \alpha / \lVert \alpha\rVert_1$ \hfill (\textit{normalize to a distribution})
\STATE $\mathcal{R}_{\mathrm{struct}}\leftarrow \mathsf{H}(\bar{A}')\;=\;-\sum_{t=1}^T \bar{A}'_t\log \bar{A}'_t$
\STATE $\mathcal{R}_{\mathrm{sem}}\leftarrow \exp\!\Big(-\frac{1}{|\Omega|}\sum_{j\in\Omega}\log P(x_j\mid x_{<j})\Big)$
\STATE $(b_{\mathrm{ent}},b_{\mathrm{ppl}})\leftarrow \textsc{Discretize}(\mathcal{R}_{\mathrm{struct}},\mathcal{R}_{\mathrm{sem}})$

\vspace{0.35em}
\STATE \textbf{(B) Stage 1: Stabilized token utility (vectorized)}
\FOR{$l=1$ \textbf{to} $L$}
  \FOR{$h=1$ \textbf{to} $H$}
    \STATE $r^{(l,h)} \leftarrow \lVert V^{(l,h)}\rVert_{2,\mathrm{row}}\in\mathbb{R}^{T}$ \hfill (\textit{row-wise $\ell_2$ norm})
    \STATE $\mu^{(l,h)} \leftarrow \textsc{Mean}_t\!\left(r^{(l,h)}\right)$
    \STATE $\rho^{(l,h)} \leftarrow r^{(l,h)} / (\mu^{(l,h)}+\epsilon) \in\mathbb{R}^{T}$
    \STATE $u^{(l,h)} \leftarrow \alpha \odot \rho^{(l,h)} \in\mathbb{R}^{T}$ \hfill (\textit{elementwise})
  \ENDFOR
\ENDFOR

\vspace{0.35em}
\STATE \textbf{(C) Stage 2: Head-aware importance injection (weighted max-pool)}
\FOR{$l=1$ \textbf{to} $L$}
  \STATE $\hat{u}^{(l)} \leftarrow \textsc{MaxPool}_{h}\!\Big(u^{(l,h)}\odot \mathbf{W}_{\mathrm{head}}[l,h]\Big)\in\mathbb{R}^{T}$
  \hfill (\textit{tokenwise over heads})
\ENDFOR

\vspace{0.35em}
\STATE \textbf{(D) Stage 3: Risk-adaptive gating + budget correction}
\FOR{$l=1$ \textbf{to} $L$}
    \STATE $\tau^{(l)} \leftarrow \mathbf{T}_{\mathrm{gate}}[l,b_{\mathrm{ent}},b_{\mathrm{ppl}},B_l]$
    \STATE $m^{(l)} \leftarrow \mathbb{I}\!\left[\hat{u}^{(l)} \ge \tau^{(l)}\right]\in\{0,1\}^{T}$ \hfill (\textit{candidate mask})
    \STATE $\mathcal{S}^{(l)} \leftarrow \textsc{Select}\big(\hat{u}^{(l)},\,m^{(l)},\,B_l\big)$
    \STATE $(\tilde{K}^{(l)},\tilde{V}^{(l)}) \leftarrow \textsc{Gather}_{\mathrm{tok}}\!\left(K^{(l)},V^{(l)};\mathcal{S}^{(l)}\right)$
\ENDFOR
\STATE \textbf{return} $\{(\tilde{K}^{(l)},\tilde{V}^{(l)})\}_{l=1}^{L}$

\vspace{0.25em}
\STATE \textit{// \textsc{Select}: if $\|m^{(l)}\|_0\le B_l$, keep all candidates; else return Top-$B_l$ indices under score $\hat{u}^{(l)}$.}
\STATE \textit{// \textsc{Gather}$_{\mathrm{tok}}$ gathers along token dimension; $\mathcal{S}^{(l)}$ is shared across heads within layer $l$.}
\end{algorithmic}
\end{algorithm}

% For Stage 2, we use a head scaling range $W_{\text{head}}[l,h]\in[0.8,1.5]$.
% We intentionally restrict the minimum weight to $0.8$ so that no head is fully suppressed, avoiding brittle behavior when a head becomes unexpectedly important for a rare prompt.
% The maximum weight $1.5$ is sufficient to create clear head heterogeneity without causing extreme re-ranking that would override the risk-gating stage.
% In practice, the \emph{effective} influence of Stage 2 is further bounded because it only rescales the shared utility $u_t$ and the final selection remains budget-feasible after Stage 3.
% We found that moderate variations within these bounded ranges produce consistent performance ordering and do not qualitatively change the learned policies, suggesting the method is not driven by fragile hyperparameter tuning.

%%%%%%%%%%%%%%%%%%%%%%%%%%%%%%%%%%%%%%%%%%%%%%%%%%%%%%%%%%%%%%%%%%%%%%%%%%%%%%%
%%%%%%%%%%%%%%%%%%%%%%%%%%%%%%%%%%%%%%%%%%%%%%%%%%%%%%%%%%%%%%%%%%%%%%%%%%%%%%%

% NeurIPS paper checklist (required for all submissions).
% Flush all appendix floats before starting the required checklist, so appendix algorithms/figures never appear on checklist pages.
% \clearpage
% \input{checklist.tex}

\end{document}